%% file: acml2013_marab.tex
\documentclass[wcp]{jmlr}

\usepackage{longtable}% for long tables

\usepackage{booktabs}
\usepackage{caption}
\jmlrvolume{29}
\jmlryear{2013}
\jmlrworkshop{ACML 2013}
\title[ Risk Aware Multi-Armed Bandits]{Exploration vs Exploitation vs Safety:\smallskip\\Risk-Aware  Multi-Armed  Bandits}

  \author{\Name{Nicolas Galichet} \Email{Nicolas.Galichet@lri.fr}\\
  \Name{Mich\`ele Sebag} \Email{Michele.Sebag@lri.fr}\\
   \Name{Olivier Teytaud} \Email{Olivier.Teytaud@lri.fr}\\
   \addr
TAO, CNRS - INRIA - LRI, Universit\'e Paris Sud, F-91405 Orsay} 
\editor{Cheng Soon Ong and Tu Bao Ho}
\begin{document}

\maketitle
\setcounter{page}{245}

%\newcommand{\cs}[1]{\texttt{\char`\\#1}}

% \usepackage{amssymb,xcolor}
% \setcounter{tocdepth}{3}
% \usepackage{graphicx}
% \usepackage[utf8]{inputenc}
% \usepackage{amsmath}
% \usepackage{amssymb}
% \let\proof\relax 
% \let\endproof\relax
% \usepackage{amsthm}
% \usepackage[english]{babel}
% \usepackage{array,longtable,calc}
% \usepackage{url}
 
% \usepackage{ifthen}
% \captionsetup{compatibility=false}
%\urldef{\mails}{{{Nicolas.Galichet, Michele.Sebag, Olivier.Teytaud}@lri.fr}}  

%\newcommand{\keywords}[1]{\par\addvspace\baselineskip
%\noindent\keywordname\enspace\ignorespaces#1}

\def\XX{{\sc MaRaB}}

\def\e{\epsilon}
\def\EE{{\rm I\hspace{-0.50ex}E}}
\def\CV{\mbox{CVaR}}

\def\version{5}

\newtheorem{prop}{Proposition}[section]
\newtheorem{cor}[prop]{Corollary}
\newtheorem{lem}[prop]{Lemma}
% histoire d'uniformiser les noms
\def\Exp{ExpExp}
\def\MV{MV-LCB}

\begin{abstract}
Motivated by applications in energy management, this paper presents the Multi-Armed Risk-Aware Bandit (\XX) algorithm. With the goal of limiting the exploration of risky arms, \XX\ takes as arm quality its conditional value at risk. When the user-supplied risk level goes to 0, the arm quality tends toward the essential infimum of the arm distribution density, and \XX\ tends toward the MIN multi-armed bandit algorithm, aimed at the arm with maximal minimal value. As a first contribution, this paper presents a theoretical analysis of the MIN algorithm under mild assumptions, establishing its robustness comparatively to UCB. The analysis is  supported by extensive experimental validation of MIN and \XX\ compared to UCB and  state-of-art risk-aware MAB algorithms on  artificial and real-world problems.
\end{abstract}

\input{introduction_camera}

\input{SoA_camera.tex}

\input{overview_camera}

\input{section_thms}

\input{exp_camera}

\section{Discussion and perspectives}

The first contribution of the paper, as a step toward an effective trade-off between exploitation, exploration and safety, is to show the theoretical 
soundness of the MIN algorithm. This result relies on two main assumptions: i) same arms are optimal in the perspective of regret 
and risk minimization; ii) the arm reward distributions are lower bounded in the neighborhood of their minimum on the other hand. 
Not only does MIN achieve logarithmic regret; it also yields a better rate than UCB under the additional assumption that min-related margins are higher than mean-related margins. 
A second contribution is the \XX\ strategy, yielding a reduced risk at the expense of a moderate regret increase compared to UCB for short and medium time horizons, 
on artifial problems (which only satisfies the lower-bounded distribution assumption) and on a real-world one.

Further work is concerned with the analysis of \XX\ behavior, specifically its mean-related regret (under the assumption that the arm with best mean also is the arm with best \CV) and its \CV-related regret; another priority is to compare \XX\ behavior with that of \cite{Maillard2013}. 
A second perspective regards the case where the arms belong to a metric space; the goal becomes to exploit this metrix to enforce exploration safety. Last but not least, \XX\ will be extended to tree-structured search spaces to achieve e.g. safe sequential decision making. 

\acks{We are grateful to J.-J. Christophe, J. Decock and the members of the Ilab Metis and Artelys, for fruitful collaboration. We thank the anonymous referees for their insightful comments.}

 %\bibliography{biblio}
% \appendix
% \input{th_analysis}
{\footnotesize
\bibliography{mybib}}
\end{document}

%% file: introduction_camera.tex
\section{Introduction}
The multi-armed bandit (MAB) framework has been intensively investigated in the last decade, handling the exploration {\em vs} exploitation dilemma through diverse selection rules, e.g. the upper confidence bound (UCB) criterion \citep{Auer:2002}, KL-UCB \citep{KL-UCB} or Thompson sampling \citep{Chappelle}. The rise of MAB studies is explained as they tackle the rigorous analysis of algorithms in a both simplified and challenging setting.   
On the one hand, the MAB framework defines a simplified reinforcement learning problem \citep{Sutton,Szepesvari}, where the state space involves a single state. 
On the other hand, MAB is concerned with lifelong learning, and the optimization of the policy return {\em while learning it}. RL classically distinguishes the learning phase, and the production phase where the learned policy is applied. Most generally MAB makes no such distinction and aims at minimizing the cumulative regret suffered compared to the oracle strategy, during the whole learning/production life of the MAB system.

This paper specifically focuses on application domains where the exploration of the environment involves hazards and risks. Examples of such domains are energy management and robotics. In robotics, policies learned using a generative model of the environment (e.g. a simulator) happen to be inaccurate when ported on the actual robot, a phenomenon known as {\em reality gap}. On the other hand, training the robot {\em in-situ} entails significant risks due to mechanical fatigue and hazards. 
In the domain of energy management, simulators define noisy optimization problems as they must reflect the large variance of the energy load, demand and price along time; in the meanwhile, exploratory policies face huge losses if insufficient supply must be compensated for by buying additional energy.  

In risky environments, the goal becomes to learn a policy which achieves some trade-off between exploration, exploitation and safety: while the goal is to minimize the regret through a standard exploration {\em vs} exploitation trade-off, one also wants to minimize the risk incurred by the learned policy and/or to maximize the safety of the learning agent. Risk minimization is also enforced by considering short time horizons.

The importance of risk minimization, either within the MAB setting \citep{Sani2012, Maillard2013} or generally in reinforcement learning \citep{MoldovanNIPS2012} has been recently acknowledged  (section \ref{sec:SoA}). A new algorithm, the multi-armed risk-aware bandit (\XX) algorithm, is presented in section \ref{sec:alg}. \XX\ aims at the arm with maximal {\em conditional value at risk} level $\alpha$ (\CV$_\alpha$), where \CV$_\alpha$ is the expected policy return in the prescribed quantile. When $\alpha$ goes to 0, \XX\ tends toward the MIN multi-armed bandit algorithm, aimed at the arm with maximal minimal value. A theoretical analysis of the MIN algorithm shows that it achieves logarithmic regret under mild assumptions (section \ref{sec:thms}). Extensive empirical validation on artificial problems shows that \XX\ less explores the arms with low distribution tails compared to UCB and \citep{Sani2012}, at the expense of a 
moderate regret increase compared to UCB. 
 A real-world problem related 
to battery management with a stochastic demand is also considered to 
investigate the robustness of the approach (section \ref{sec:expval}).
The paper concludes with a discussion and some perspectives for further research.

%% file: SoA_camera.tex
\section{State of the art}\label{sec:SoA}
After introducing the multi-armed bandit (MAB) formal background and referring the reader to \citep{Robbins:1952,Auer:2002} for a comprehensive presentation, this section briefly reviews the state of
the art related to risk-aware MAB strategies. 

\subsection{Formal background}
\label{sec:SoA_MAB}
A multi-armed bandit problem involves $K$ independent arms, each of which has an unknown reward distribution with bounded discrete or continuous support.
The literature mostly considers two settings, that of Bernoulli distributions where the $i$-th arm yields reward 1 with probability $\mu_i$ and reward 0 otherwise, and that of distributions with support $[0,1]$, with mean $\mu_i$ and standard deviation $\sigma_i$. Let $T$ denote the time horizon. 
At each time step $t = 1\ldots T$, a MAB algorithm selects an arm $i_t$
and receives reward $r_{i_t,n_{i,t}}$, drawn after the $i_t$-th distribution, where $n_{i,t}$ denotes the number of times the $i$-th arm has been selected up to time $t$. 
The choice is made upon the basis of the empirical estimates of the $K$ arm distributions so far, the empirical mean estimate $\widehat{\mu_{i,t}}$  
%% Attention ici erreurs ds les indices
($\widehat{\mu_{i,t}} = \frac{1}{n_{i,t}} \sum_{u=1}^{u=n_{i,t}}  r_{i,u}$) and possibly the empirical variance estimate $\widehat{\sigma_{i,t}}$  (${\widehat{\sigma_{i,t}}}^2 = \frac{1}{n_{i,t}} \sum_{u=1}^{u=n_{i,t}} (r_{i,u} - \widehat{\mu_{i,u}})^2$). 
%In a classical RL setting, the goal would be to identify the arm $i^*$ with best expectation after $T$ trials. 
The MAB goal is to maximize the sum of gathered rewards along learning, or equivalently to minimize the cumulative regret suffered compared to the oracle strategy, which plays the best arm $i^*$ in each time step. 
One distinguishes the theoretical cumulative regret at time $t$, denoted $\mathcal{R}_t$, and the empirical cumulative regret, denoted $\widehat{\mathcal{R}_t}$, respectively defined as:
\begin{equation*}
\mathcal{R}_t = t \times \mu_{i^*} - \sum_{k=1}^K n_{k,t} \mu_{k} \hspace{1in}  
\widehat{\mathcal{R}_t}= t \times \mu_{i^*} - \sum_{k=1}^K n_{k,t} \widehat{\mu_{k,t}}
\end{equation*}
We will use the theoretical cumulative regret in the algorithm analysis (section \ref{sec:thms}). Theoretical or empirical cumulative regrets will be used to experimentally assess the algorithms performance (section \ref{sec:expval}),  granted that the difference $|\mathcal{R}_t - \widehat{\mathcal{R}_t} |$ is in $O(\log(t))$ \citep{coquelin}.

Regret minimization is known to be an {\em exploitation vs exploration} trade-off problem: the best empirical arm should be selected often to maximize the actual gathered reward (exploitation); but some exploration is also required to actually identify the best arm.
Two prominent MAB strategies are the $\epsilon$-greedy strategy, which selects the best empirical arm with probability $1 - \epsilon$ and uniformly selects another arm with probability $\epsilon$, and the celebrated upper confidence bound (UCB) strategy proposed by \cite{Auer:2002}, which selects at time $t$ the arm maximizing criterion 
$ \widehat{\mu_{i,t}} + C \sqrt{\frac{log(t)}{n_{i,t}}}$,
with $C > 0$ a parameter controlling the exploration {\em vs} exploitation tradeoff. Another strategy is the KL-UCB strategy \citep{KL-UCB}.  
While the $\epsilon$-greedy strategy suffers a linear regret, UCB and KL-UCB suffer a logarithmic regret, which is known to be optimal  \citep{LaiRobbins:1985}. KL-UCB further improves on UCB as it yields the optimal regret rate.  

\subsection{Related work}\label{sec:defcvar}
An emerging trend in the field of reinforcement learning and MAB 
is concerned with the risk issue, when the exploring agent might face 
hazards going beyond mere under-optimal performances. In such cases, mottos such as {\em Optimism in front of the Unknown!}~attached to the UCB strategy, are inappropriate. 
A first issue concerns the definition of risks. Several definitions have been proposed to account for risk awareness and risk aversion, taking inspiration from the literature in economics \citep{Arrow71}.

The first criterion referred to as mean-variance (MV) \citep{Markowitz52} considers a weighted sum of the reward expectation $\mu$ of the policy and its estimated standard deviation. Formally, the goal is to find a policy minimizing  $\sigma^2 - \rho\mu$, where $\rho > 0$ increases like the user's {\em risk tolerance}.

The conditional value at risk (\CV) considers the quantiles of the reward distribution. Formally, let $0 < \alpha < 1$ be the target quantile level. 
The associated quantile value $v_\alpha$ is defined if it exists as the maximal value such that $X$ is less than $v_\alpha$ with probability $\alpha$ ($Pr(X < v_\alpha) = \alpha$, with $X$ the reward random variable). The remainder of the paper only considers continuous distributions, where $v_\alpha$ is always defined; the conditional value at risk $\alpha$ noted $\CV_\alpha$ is then defined as the average reward conditionally to $X < v_\alpha$:
%\begin{equation}
%\label{CVaR:def}
 $\CV_\alpha [X]= \EE[X | X < v_\alpha]$.
% \end{equation} 
Note that when the quantile level $\alpha$ goes to 0, \CV$_\alpha$ maximization coincides with the standard max-min strategy, aimed at the arm with maximal minimum reward. \CV\ maximization thus defines a relaxation of the max-min strategy, with quantile level $\alpha$ as relaxation parameter. 

Another criterion, the rank dependent utility \citep{RDU} inspired from the prospect theory due to Tversky and Kahneman (1979), is meant to model the distorted perception of probabilities, e.g. the over-estimation of rare events, through weighting the event rewards with a (non-linear) function of their rank. The RDU criterion will not be considered further as it relies on a complex specification of the risk aversion, whereas the above MV and \CV\ criteria involve a single scalar parameter, respectively $\rho$ and $\alpha$.

Risk aversion has been considered in the MAB setting, only tackling the mean-variance criterion to our best knowledge \citep{Sani2012}. Two algorithms are proposed. The first one, referred to as \MV, aims at minimizing the MV cumulative regret. It proceeds by adapting the UCB approach in the finite horizon $T$ context, selecting in each step $t$ the arm maximizing 
%\label{mvreward}
$ \widehat{\sigma_{i,t}}^2 - \rho \widehat{\mu_{i,t}} - (5 + \rho)
	   \sqrt{\frac{log(\frac{1}{\delta})}{2 n_{i,t}}}$,
%\end{equation}
where $\delta$ is adjusted depending on the time horizon $T$. As shown by 
\cite{Sani2012}, this selection rule leads to upper-bounding the theoretical cumulative regret (related to the MV criterion)  $R_t/t$ by $O(\frac{\log^2(t)}{t})$.
% TODO here and below, this is a bound on $R_n/n$, and not a bound on $R_n$, isn't it ?
A simpler strategy referred to as \Exp\ decouples the exploration and the exploitation phases. All arms are uniformly launched during the 
exploration phase, and the arm with optimal empirical MV is selected ever after during the exploitation phase, with an $O( K T^{-\frac{1}{3}})$ regret bound  if the length $\tau$ of the exploration phase is fixed to $K(\frac{T}{14})^{2/3}$.

Risk issues have also been considered in the neighbor field of reinforcement learning in the last years. A first strategy relies on reversibility constraints, i.e. only visiting states $s$ such that one can always get back from $s$ to the initial state \citep{MoldovanA12}. In a  further work, 
\cite{MoldovanNIPS2012} proceed by considering an exponential utility function, where the policy return  $J$ is replaced by expression $exp\{J/\theta\}$. Parameter $\theta$ reflects the user's risk tolerance, akin the $\rho$ parameter in the mean var setting, with the difference that $\rho$ is weighted by the empirical standard deviation of the 
rewards. 

Another approach due to \cite{Mannor2011} formalizes risk-aware reinforcement learning as a 
multi-objective RL problem, aimed at simultaneously maximizing the cumulative reward and minimizing the cumulative standard deviation.

%% file: overview_camera.tex
\section{Overview of \XX}\label{sec:alg}
This section describes the {\em Multi-Armed Risk-Aware Bandit} (\XX) algorithm, with same notations as in section \ref{sec:SoA_MAB}.
\def\CVj{\mbox{$\widehat{CVaR_{\alpha}}$}}
\def\cvj{\mbox{$\widehat{CVaR}$}}

\def\CVi{\mbox{$\widehat{CVaR_{\alpha,i}}$}}
\def\cvi{\mbox{$\widehat{CVaR_i}$}}
\def\ta{\mbox{$\lceil \alpha n_{i,t} \rceil$}}
\def\nit{\mbox{$max(1,\lceil \alpha n_{i,t} \rceil)$}}
\def\nia{\mbox{$n_{i,t,\alpha}$}}

\def\nja{\mbox{$n_{t,\alpha}$}}

The arm quality is set to its conditional value at risk $\alpha$, where parameter $\alpha$ ($0 < \alpha < 1$) is set by the user. %Let $n_t$ denote the number of times the arm has been selected up to the $t$-th iteration, with $r_{u},~ u=1\ldots n_t$ the associated rewards.
After \cite{chen}, a non-parametric, consistent estimate of the conditional value at risk $\alpha$ of arm $i$, denoted \CVi\ (or \cvi\ for notational simplicity), is given as the average of the $\alpha$ quantile of rewards $r_{i,u}, u=1\ldots n_{i,t}$: assuming with no loss of generality that rewards are ordered by increasing value ($r_{i,u} \le r_{i,u+1}$), and noting \nia\ the ceiling integer of $\alpha \cdot n_{i,t}$ ($\nia = \lceil \alpha n_{i,t}\rceil$), then \cvi\ is set to the average of the lowest $n_{i,t,\alpha}$ rewards:
\begin{equation}
  \cvi = \frac{1}{\nia}\sum_{u=1}^{\nia} r_{i,u}
\label{eq:cvv}
\end{equation}

The goal of \XX\ is to find the arm with maximal \cvi. The selection rule controlling the exploration {\em vs} exploitation tradeoff proceeds by selecting the arm with best lower confidence bound on its \CV: 
\begin{equation}
  \mbox{select $i_t$~} = \mbox{argmax} \left \{\cvi - C \sqrt{\frac{\log(\lceil t\alpha\rceil)}{\nia}} \right \}
\label{eq:xx}
\end{equation}
with $C > 0$ a parameter controlling the exploration {\em vs} exploitation tradeoff.

\XX\ features a risk-averse or pessimistic behavior, due to the negative exploratory term in Eq. \ref{eq:xx}\,: if two arms have same empirical \CV, \XX\ will favor the arm which has been selected more often in the past. Note that such a behavior is actually observed in the economic realm, as trust $-$ i.e. a positive bias toward known good partners $-$ is at the core of economic exchanges. Such a bias indeed makes sense whenever
exchanges with unknown partners involve risks. 

A lack of exploration usually leads to myopic and under-optimal choices, 
sticking to the best options first encountered. Such a myopic behavior is however prevented in \XX\ for the following reason:
\XX\ examines each arm along two phases. 
In the first phase, referred to as initial phase ($n_{i,t} < \frac{1}{\alpha}$), the empirical 
quality of the $i$-th arm is set to its minimum reward (Eq. \ref{eq:cvv}), and therefore it monotonically decreases along time. In the second phase, referred to as stabilization phase, the estimate of the conditional value at risk is computed with increasing accuracy, with an approximation error going to 0 like $\sqrt{n_{i,t}}$ \citep{chen}. 

The duration of the initial phase increases as $\alpha$ decreases, as the maximization of the conditional value at risk $\alpha$ boils down to a standard max-min optimization problem. In the early iterations,  the \XX\ behavior thus coincides with that of the MIN algorithm, selecting in each time step the arm with maximal minimum reward. The only difference comes from the negative exploration term (lower confidence bound, LCB\footnote{This LCB must not be mistaken for the LCB used in MV-LCB \citep{Sani2012}, section \ref{sec:SoA}: as the MV-LCB reward is the weighted sum of the average standard deviation and means, where the weight of the empirical mean is negative, this LCB actually behaves as a UCB, optimistically favoring the exploration.}). 
\XX\ thus actually achieves some exploration in the beginnings: the 
arm quality \cvi\ monotonically decreases as the $i$-th arm is more visited ($n_{i,t}$ increases) in its initial phase, forcing \XX\ to consider less visited arms. 

However, if an arm gets poor rewards the first times it is visited, there is little chance it is visited again, all the more so as better arms 
enter their second phases (and their empirical quality converges toward their true conditional value at risk): there is no positive exploration term guaranteeing that any arm will be visited infinitely many times as $t$ goes to infinity. 

The theoretical analysis, presented in the next section, will thus focus on the limit algorithm of \XX, the MIN algorithm.

%% file: section_thms.tex
\section{Analysis}\label{sec:thms}
This section presents the analysis of the MIN algorithm, selecting in each time step the arm with maximal empirical minimal value, as  MIN is the limit algorithm of \XX\ when the risk level $\alpha$ goes to 0 and the exploratory constant $C$ is set to $0$. 

Under the assumption that the best arm w.r.t. its essential infimum also is the best arm in terms of expectation, it is shown that MIN achieves same logarithmic regret as UCB, with similar rate. Under slightly stronger assumptions, the MIN regret rate is significantly lower than for UCB. These two results rely on two lemmas. Firstly, under mild assumptions, the empirical minimum value for every arm converges exponentially fast toward its essential infimum. Secondly, with high probability over all arms their empirical minimum value are exponentially close to their essential minimum, where the probability increases exponentially fast with the number of iterations. 

\begin{lem}\label{lem41}
Let $\nu$ be a bounded distribution with support in $[0,1]$, with $a$ its 
essential infimum\footnote{The essential infimum being defined as the maximal value $a$ such that $\mathbb{P}(X < a) = 0$.}, and let us assume that $\nu$ is lower bounded in the neighborhood of $a$:
\begin{equation}
 \exists A > 0, \forall \epsilon > 0, \mathbb{P}(X \leq a + \epsilon) \geq A\e \mbox{~ with $X$ r.v. $\sim \nu$}
\label{eq:lowerboundA}
\end{equation}

% Je remplace par une minuscule
% by convention capital letters will represent the random variable while lowercase letters will represent the actual realization of the ...
% in https://mywebspace.wisc.edu/jahlquist/web/ProbabilityII.pdf
Let $x_1 \ldots x_t$ be a t-sample independently drawn after $\nu$. 
Then, the minimum value over $x_{u}, u =1 \ldots t$ goes exponentially fast to $a$:
\begin{equation}
 \mathbb{P}(\underset{1\leq u \leq t}{\min} x_{u} \geq a + \epsilon) \leq \exp(-tA\epsilon)
\label{eq:lem1}
\end{equation}
\end{lem}
\begin{proof}\\
As the $x_u$ are iid, it comes:
\begin{eqnarray*}
\mathbb{P}(\underset{1\leq u \leq t}{\min} x_{u} \geq a + \epsilon) & = &  \mathbb{P}(\forall u \in \{ 1,\ldots,t\}, x_{u} \geq a + \epsilon)\\
		& = & \prod_{u=1}^t \mathbb{P}( x_{u} \geq a + \epsilon)
		 \leq  (1-A\epsilon)^t 
      \leq  exp (-t A \epsilon)
\end{eqnarray*}
where the last inequality follows from $(1 - z) \leq exp(-z)$. 
\end{proof}
Under the assumption of a lower-bounded distribution probability in the neighborhood of its minimum, the convergence toward the minimum thus is faster than the convergence toward the mean. Specifically, the Hoeffding bound on the convergence toward the mean decreases exponentially like $-t \epsilon^2$, whereas after Eq. \ref{eq:lem1} the convergence toward the min decreases exponentially like\footnote{The convergence analysis considers an approximation error $\epsilon$ going to 0, hence $A \epsilon >> \epsilon^2$.} $-t A \epsilon$.\\

Under the same assumptions, with high probability the empirical min of each arm is exponentially close to its essential infimum after each arm has been tried $t$ times. 

\begin{lem}\label{firstprop}
 Let  $\nu_1 \ldots \nu_K$ denote $K$ distributions with bounded support in $[0,1]$ with $a_i$ their essential infimum. Let us assume that $\nu_i$ is lower bounded by some constant $A$ in the neighborhood of $a_i$ for $i=1 \ldots K$. \\
Denoting $x_{i,u}$, $~u=1 \ldots t$, $i = 1\ldots K$, $t$ samples independently drawn after $\nu_i$, one has:\\ 
\begin{equation}
\label{bounddeviation}
\mathbb{P}(\exists i \in \{1,\ldots,K \}, \underset{1\leq u \leq t}{\min} x_{i,u} \geq a_i + \epsilon) \leq K \exp(-tA \epsilon )
\end{equation}
\end{lem}
\begin{proof}
After Lemma \ref{lem41}, 
\begin{eqnarray*}
\mathbb{P}(\exists i \in \{1,\ldots, K\}, \underset{1\leq u \leq t}{\min} x_{i,u} \geq a_i + \epsilon)  &\leq&  1 - ( 1 - (1 -A \epsilon)^t)^K \\
&\leq&  K (1 - A \epsilon)^t \\
&\leq&  K\exp(-t A\epsilon)
\end{eqnarray*}
Where the first inequality follows from $(1 - z)^y \ge 1 - y.z$ and the second inequality from $(1 - z) \le exp(-z)$, which concludes the proof.
\end{proof}

Under the above assumptions on the arm distributions, if the optimal arm in terms of min value also is the optimal arm in terms of mean value, then the MIN algorithm achieves a logarithmic regret. 
% empirical cumulative regret
\def\ecr{\widehat{\mathcal{R}_t}}
% true cumulative regret
\def\tcr{{\mathcal{R}_t}}

\begin{prop}\label{xx}\label{prop43}
Let  $\nu_1 \ldots \nu_K$ denote $K$ distributions with bounded support in $[0,1]$ with $\mu_i$ (resp. $a_i$) their mean (resp. their essential infimum). Let us further assume that $\nu_i$ is lower bounded by some constant $A$ in the neighborhood of $a_i$ for $i=1 \ldots K$, and that the arm with best mean value $\mu^*$ also is the arm with best min value $a^*$. 
Let $\Delta_{\mu,i} = \mu^* -\mu_i$ (resp. $\Delta_{a,i} = a^* -a_i$) denote the mean-related (resp. essential infimum-related) margins.
Then, with probability at least $1-\delta$, the cumulative regret is upper bounded as follows:
\begin{equation}\label{regretboundprob}
\tcr \leq \frac{K-1}{A} \frac{ \Delta_{\mu,\max} }{\Delta_{a,\min}} \log\left(\frac{t K}{\delta}\right)+  (K-1) \Delta_{\mu, \max} 
\end{equation}
with $\Delta_{a,\min} = \underset{i}{\min} \, \Delta_{a,i}$ and $\Delta_{\mu,\max} = \underset{i}{\max} \, \Delta_{\mu,i}$.

Furthermore, the expectation of the cumulative regret is upper-bounded as follows for $t$ sufficiently large ($t \geq \frac{K-1}{A}\frac{\Delta_{a,\min}}{\Delta_{\mu,\max}}$):
\begin{equation}\label{eq:minexp}
%\label{regretbound}
\mathbb{E}[\tcr] \leq \frac{K-1}{A} \frac{\Delta_{\mu,\max}}{\Delta_{a,\min}} \left(\log\left(\frac{t^2KA}{K-1} \frac{\Delta_{a,\min}}{\Delta_{\mu,\max}} \right)+1\right) +(K-1) \Delta_{\mu,\max}
\end{equation}
\end{prop}
%{\huge la preuve est assez inspiree de celle de Sani et al., faut- il le signaler ? Oui, ds l'appendix}
\begin{proof}
Let us assume that there exists a single optimal arm (we shall return to this point below).
Taking inspiration from \cite{Sani2012}, let $x_{i,u}$ be independent samples drawn after $\nu_i$, and define the event set $\mathcal{E}$ as follows:
\begin{equation}
\label{eventsetdef}
\mathcal{E} = \{ \forall i \in \{ 1,\ldots,K\}, \forall u \in \{1,\ldots,t\}, \underset{1\leq s \leq u}{\min} x_{i,s} - a_i \leq  \frac{\epsilon}{u}  \}
\end{equation}
The probability of the complementary set $\mathcal{E}^c$ is bounded  after Lemma \ref{firstprop}:
\begin{eqnarray*}
\mathbb{P}(\mathcal{E}^c) & = & \mathbb{P}(\exists i \in \{ 1,\ldots,K\}, \exists u \in \{ 1,\ldots, t \},  \underset{1\leq s \leq u}{\min} x_{i,s} - a_i > \frac{\epsilon}{u}  )\\
& \leq & \sum_{u=1}^t \mathbb{P} (\exists i \in \{ 1,\ldots, K\}, \underset{1\leq s \leq u}{\min} x_{i,s} - a_i > \frac{\epsilon}{u} )\\
& \leq & \min(1, t K \exp(-A\epsilon))
\end{eqnarray*}

Let $t > 1$ be an iteration where a sub-optimal arm $i$ is selected; this implies that the empirical min of the $i$-th arm is higher than that of the best arm $i^*$:
\begin{eqnarray*}
 \underset{1\leq u \leq n_{i^*,t-1}}{\min} x_{i^*,u} <  \underset{1\leq u \leq n_{i,t-1}}{\min} x_{i,u} & \Leftrightarrow&\underbrace{\underset{1\leq u \leq n_{i^*,t-1}}{\min} x_{i^*,u} - a_i}_{\geq a_{i^*} - a_i = \Delta_{a,i}} <  \underbrace{\underset{1\leq u \leq n_{i,t-1}}{\min} x_{i,u} - a_i}_{\leq \frac{\epsilon}{n_{i,t-1} }(*)}\\
\end{eqnarray*}
where $(*)$ holds if $t$ belongs to the event set $\cal E$, thus with probability at least $1 - tK exp(-A\epsilon)$ after Lemma \ref{firstprop}.\\
It follows that with probability at least $1 - tK exp(-A\epsilon)$
\[ \frac{\epsilon}{n_{i,t-1}} \geq \Delta_{a,i}  \mbox{~hence~} n_{i,t} \leq \frac{\epsilon}{\Delta_{a,i}}  + 1\]
since $n_{i,t} \le  n_{i,t-1} + 1$. With probability at least $1 - tK exp(-A\epsilon)$, the cumulative regret $\tcr$ can thus be upper-bounded:
\begin{eqnarray}
\tcr 
 & = &  \sum_{i=1}^K n_{i,t} \Delta_{\mu,i} \leq  \sum_{i=1}^K (\frac{\epsilon}{\Delta_{a,i}} + 1) \Delta_{\mu,i} \label{eq:further}\\
 & \leq & (K-1) \left ( \frac{\Delta_{\mu,max}}{\Delta_{a,min}}\epsilon + \Delta_{\mu, \max}\right ) \mbox{ with } \Delta_{\mu,\max} = \underset{1\leq i \leq K}{\max} \Delta_{\mu,i} \mbox{~and~} \Delta_{a,\min} = \underset{1\leq i \leq K}{\min} \Delta_{a,i} \nonumber
\end{eqnarray}

Finally, by setting $\delta = \min(1,t K \exp(-A\epsilon))$, it follows that with probability $1 - \delta$, 
\begin{equation}
\tcr \leq \frac{K-1}{A} \frac{\Delta_{\mu,max}}{\Delta_{a,min}}\log(\frac{t K}{\delta})+  (K-1) \Delta_{\mu,\max} 
\label{prop43_1}
\end{equation}

In the case where there exists $k>1$ optimal arms, Eq. \ref{prop43_1} still holds, by replacing $K-1$ factor with $K-k$.

The expectation of the cumulative regret is similarly upper-bounded: 
\begin{eqnarray*}
\mathbb{E}[\tcr] &=& \mathbb{E}[\mathcal{R}_t \mathbb{I}_{\mathcal{E}}] +  \mathbb{E}[\mathcal{R}_t \mathbb{I}_{\mathcal{E}^c}]\\
			& \leq & \frac{K-1}{A} \frac{\Delta_{\mu,max}}{\Delta_{a,min}}\log(\frac{t K}{\delta})+  (K-1) \Delta_{\mu,\max} + \delta t \mbox { by bounding }\tcr\ \mbox{ by } t \mbox{ over } \mathcal{E}^C. \\
\end{eqnarray*}

For $t$ sufficiently large ($t \geq \frac{K-1}{A}\frac{\Delta_{\mu,max}}{\Delta_{a,min}} $), by setting $\delta=\frac{K-1}{ t A} \frac{\Delta_{\mu,max}}{\Delta_{a,min}} $, it comes :  

\begin{equation}
\mathbb{E}[\tcr] \leq \frac{K-1}{A} \frac{\Delta_{\mu,max}}{\Delta_{a,min}} \left(\log\left(\frac{t^2KA}{(K-1)}\frac{\Delta_{a,\min}}{\Delta_{\mu,\max}}\right)+1\right) +(K-1) \Delta_{\mu,\max}
\end{equation}
which concludes the proof.
\end{proof}

{\noindent\bf Remark}. UCB similarly achieves a logarithmic regret  \citep{Auer:2002}:  
\begin{equation}\label{eq:ucbexp}
\mathbb{E}[\mathcal{R}_t] \leq 8 \sum_{i\neq i^*} \frac{\log t}{\Delta_{\mu,i}} + (1 + \frac{\pi^2}{3})\sum_{i=1}^K \Delta_{\mu,i}
\end{equation}
where $i^*$ stands for the index of the optimal arm. MIN and UCB thus both achieve a logarithmic regret uniformly over $t$, where the regret rate involves the mean-related margin in UCB (resp. the min-related margin in MIN, multiplied by the lower-bound constant $A$ on the density in the neighborhood of the minimum). 

A stronger result can be obtained for MIN, under an additional assumption
on the lower tails of the arm distributions. 

\begin{prop}\label{prop44}
With same notations and assumptions as in Prop. \ref{prop43}, let us further assume that for every $i = 1\ldots K, \Delta_{\mu,i} = \mu^* - \mu_i \leq a^* - a_i = \Delta_{a,i}$.\\

Then, with probability at least $1-\delta$,
\begin{equation*}
\tcr \leq \frac{K-1}{A} \log(\frac{tK}{\delta})+  (K-1) \Delta_{\mu,\max} 
%\label{regretboundprob2}
\end{equation*}
with $\Delta_{\mu,\max} = \underset{i}{\max} \, \Delta_{\mu,i}$.

Furthermore, if $t > \frac{K-1}{A}$, the expectation of $\tcr$ is upper-bounded as follows :
\begin{equation}
\label{eq:minexp2}
\mathbb{E}[\tcr] \leq \frac{K-1}{A} \left(\log\left(\frac{t^2KA}{K-1}\right)+1\right) +(K-1) \Delta_{\mu,\max}
\end{equation}
\end{prop}
\begin{proof}
The proof closely follows the one of Prop. \ref{prop43}, noting that in Eq. \ref{eq:further} $\Delta_{a,i}$ is now greater than $\Delta_{\mu,i}$. Setting $\delta = \frac{(K-1)}{tA}$ concludes the proof of Eq. 
\ref{eq:minexp2}.
\end{proof}

{\noindent\bf Discussion}. The comparison of Eq. \ref{eq:minexp2} and 
Eq. \ref{eq:ucbexp} suggests that MIN might outperform UCB in the case where  margins $\Delta_{\mu,i}$ are small, where distributions $\nu_i$ are not too thin in the neighborhood of the essential infimum (that is, $A$ is not too small), and the assumption $\Delta_{a,i} \geq \Delta_{\mu,i}$ holds. \\
Note that the latter assumption boils down to considering that better arms (in the sense of their mean) also have a narrower support for their 
lower tail, thus a lower risk. If this assumption does not hold however, then risk minimization and regret minimization are likely to be conflicting objectives. 

A last remark is that the assumptions done (lower bounded distribution density in the neighborhood of the essential minimum and mean-related margin greater than the minimum-related margin) yield a significant improvement compared to the continuous distribution-free case, where the optimal regret is known to be $O(\sqrt{ t})$ \citep{COLT09a,JMLR10}.

%% file: exp_camera.tex
\section{Experimental validation}\label{sec:expval}
As proof of concept, UCB, MIN and \XX\ are first compared on favorable cases, using a problem generator satisfying the assumptions done in Prop \ref{prop44}. 
A general empirical validation follows, assessing MIN and \XX\ comparatively to UCB and to the risk-aware \MV\ and \Exp\ algorithms \citep{Sani2012}. Artificial problem instances are generated using a relaxed problem generator, which only satisfies the assumption of lower-bounded densities in the neighborhood of their minimum 
 (section \ref{sec:artexp}). A simplified real-world problem in the target application domain of energy management is also considered (section \ref{sec:energy}). 
The goal of experiments is to answer three questions. The first one is
the price to pay in terms of performance loss for a risk-aware behavior, 
and how the cumulative regret increases with the number of iterations, specifically focussing on short time horizons (unless explicitly specified, the empirical cumulative regret is considered).
%%%%%%%%% empirical, theoretical. verifier 
The second question regards the robustness of the algorithms, and their sensitivity w.r.t. parameters. A third question is whether \XX, \MV\ and \Exp\ do avoid exploring risky arms; 
this question is investigated by inspecting the low tail of the gathered rewards.  

The number $K$ of arms is set to 20. The time horizon is set to $T = K \times 100$ and $T = K \times 200$. For all problems, all results over (respectively the average result out of) 40 runs are displayed.

\input{toy_camera}

\subsection{Artificial problems}\label{sec:artexp}
A second problem generator is considered, which only satisfies the assumption of a lower-bounded density in the neighborhood of the minimum (Eq. \ref{eq:lowerboundA}). Specifically, each problem involves  20 arms. 
The $i$-th arm distribution $\nu_i$ is set to a mixture of truncated Gaussians: i) its minimum $a_i$ is uniformly drawn in $[0, .05]$; ii) $n_i$ Gaussians are defined where $n_i$ is uniformly drawn in $1 \ldots 4$; for $j = 1\ldots n_i$  the $j$-th Gaussian ${\cal N}(\mu_{i,j},\sigma_{i,j})$, is defined by uniformly sampling $\mu_{i,j}$ in $[0,1]$ and $\sigma_{i,j}$  in $[.12, .5]$; furthermore, the $j$-th Gaussian is
 associated a probability $p_{i,j}$  such that $\sum_j p_{i,j}=1$.
Upon selecting the $i$-th arm, the reward is drawn by: i) selecting the $j$-th Gaussian with probability $p_{i,j}$; ii) drawing a reward $r$ from ${\cal N}(\mu_{i,j}, \sigma_{i,j})$; iii) going to i) if $r$ is not in the $[a_i,1]$ interval (rejection-based truncation).

\subsubsection{Cumulative regrets}
The empirical cumulative regrets of UCB, \XX, \MV\ and \Exp\ are displayed in Fig. \ref{bestOptions}, reporting the empirical cdf\footnote{For each algorithm the cumulative regrets $R[i], i = 1\ldots 1,000$ over the 1,000 problem instances are independently sorted and the curve $(i,R[\sigma(i)])$ is displayed.}  of the regrets over 1,000  problem instances for short time (Fig. \ref{bestOptions}.(a)) and medium time (Fig. \ref{bestOptions}.(b)) horizons. All algorithm parameters are set to their best value after preliminary experiments. 
UCB yields the best cumulative regret overall whenever $C$ is well tuned. %confirming its sensitivity to parameter $C$. 
% ($C = 10^{-3}$ is the optimal value for the considered problems); 
%the UCB behavior for inappropriate $C$ values is significantly worse (results omitted due to space limitations). 
\XX\ suffers an extra regret compared to UCB; this extra regret is bounded 
in the considered experimental setting, and it seemingly does not increase as the time horizon increases. As could have been expected this extra regret decreases as $\alpha$ increases and the selection rule involves a better estimation of the empirical means. Interestingly, \XX\ shows a very low sensitivity w.r.t. $C$.  

\MV\ yields the worst regret of all strategies, with a very low sensitivity w.r.t. parameter $\rho$ on the considered problems. 
\Exp\ significantly improves on \MV\ with probability circa 90\%; it even improves on UCB with probability 10\% (circa 20\% for medium time horizon).
\Exp\ yields very good results; the fact that it does never get very low 
cumulative regret is explained from its initial exploratory phase; a caveat is that its optimal setting used in the experiments requires the time horizon to be known in advance. 
\XX\ improves on \Exp\ with probability 70\%, albeit with maximal cumulative regrets (over the problem instances) higher than for \Exp. 

Overall, \XX\ with risk level $\alpha = 20\%$ and untuned $C$ value
yields results slightly less than UCB with tuned $C$, for both short and medium time horizons. The risk-aware \XX\ suffers a low regret increase compared to risk-neutral UCB, with a very low sensitivity w.r.t. $C$. Interestingly, a twice longer time horizon does not modify the performance order of the algorithms. 
\begin{figure}[htbp]
\begin{center}
\begin{tabular}{cc}
\includegraphics[width=.4\textwidth]{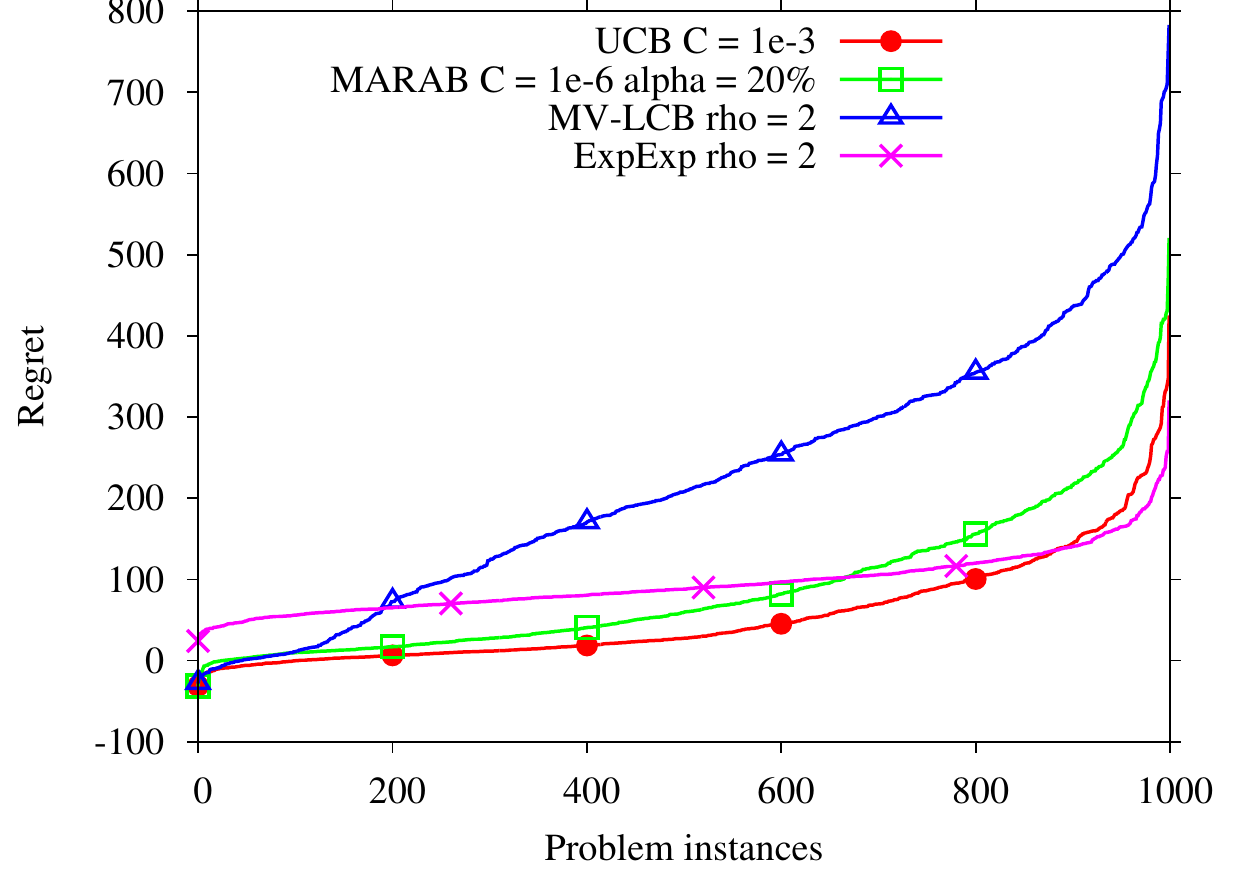} & \includegraphics[width=.4\textwidth]{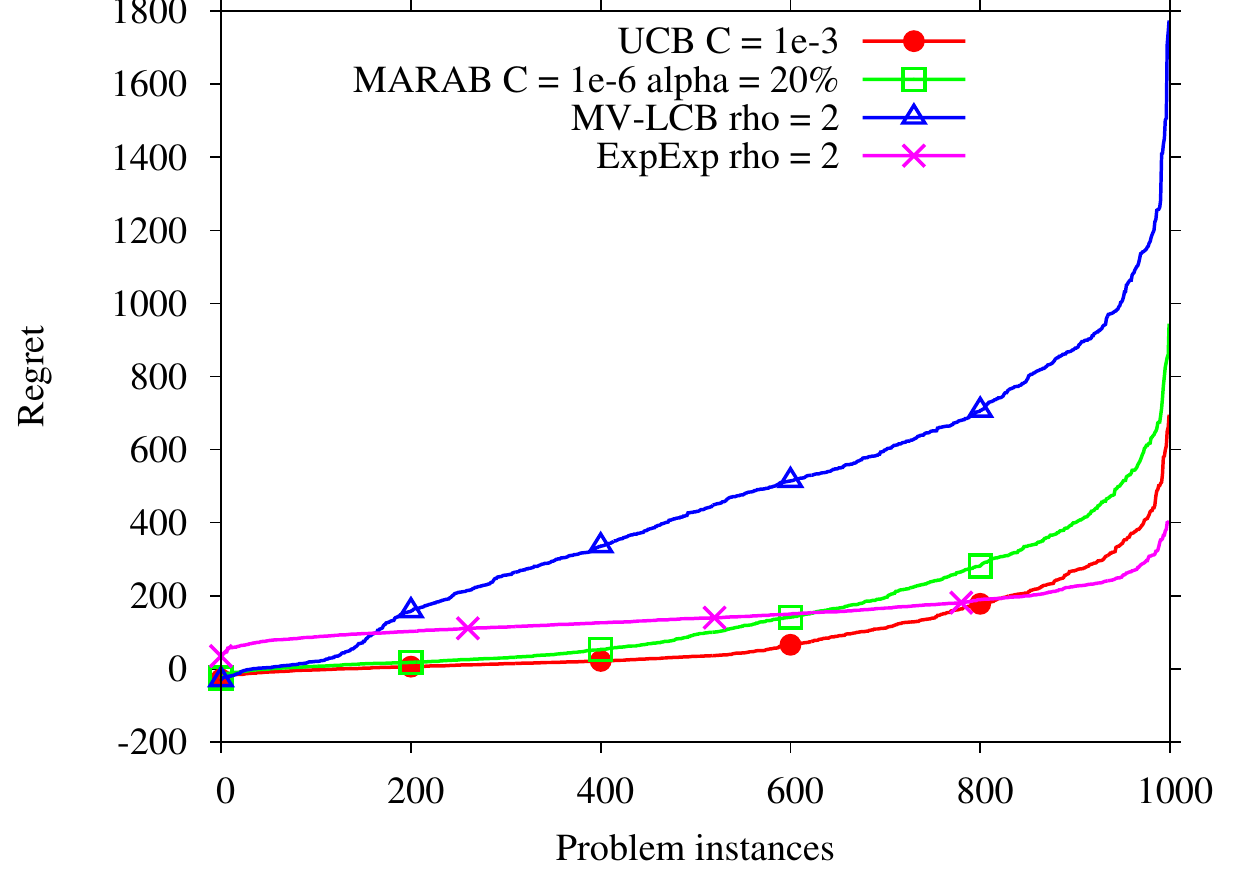}\\
(a) Time horizon 2,000 & (b) Time horizon 4,000\\
\end{tabular}
\end{center}
\caption{Empirical cumulative regret of UCB, \XX, \MV\ and \Exp\ on 1,000 problem instances (independently sorted for each algorithm) over short and medium time horizons. All algorithms are used with tuned parameters ($C = 10^{-3}, \alpha= 20\%, \rho = 2, \delta = \frac{1}{T^2}, \tau = K(\frac{T}{14})^{2/3}$). }
\label{bestOptions}
\end{figure}

\subsubsection{Risk Awareness} 
The effective risk avoidance of UCB, \MV, \Exp\ and \XX\ are investigated 
by inspecting the empirical cdf\footnote{For each algorithm the rewards $\bar{r}_t$ averaged out of 40 runs with time horizon $T=2,000$ are sorted by increasing value and the curve $(t, \bar{r}_{\sigma(t)})$ is displayed.} of the instant rewards on two representative artificial problems, with respectively low (Fig. \ref{empRewards}, left) and high  (Fig. \ref{empRewards}, right) variance of the best arm. 
\begin{figure}[htbp]
\begin{center}
\begin{tabular}{cc}
\includegraphics[width=.4\textwidth]{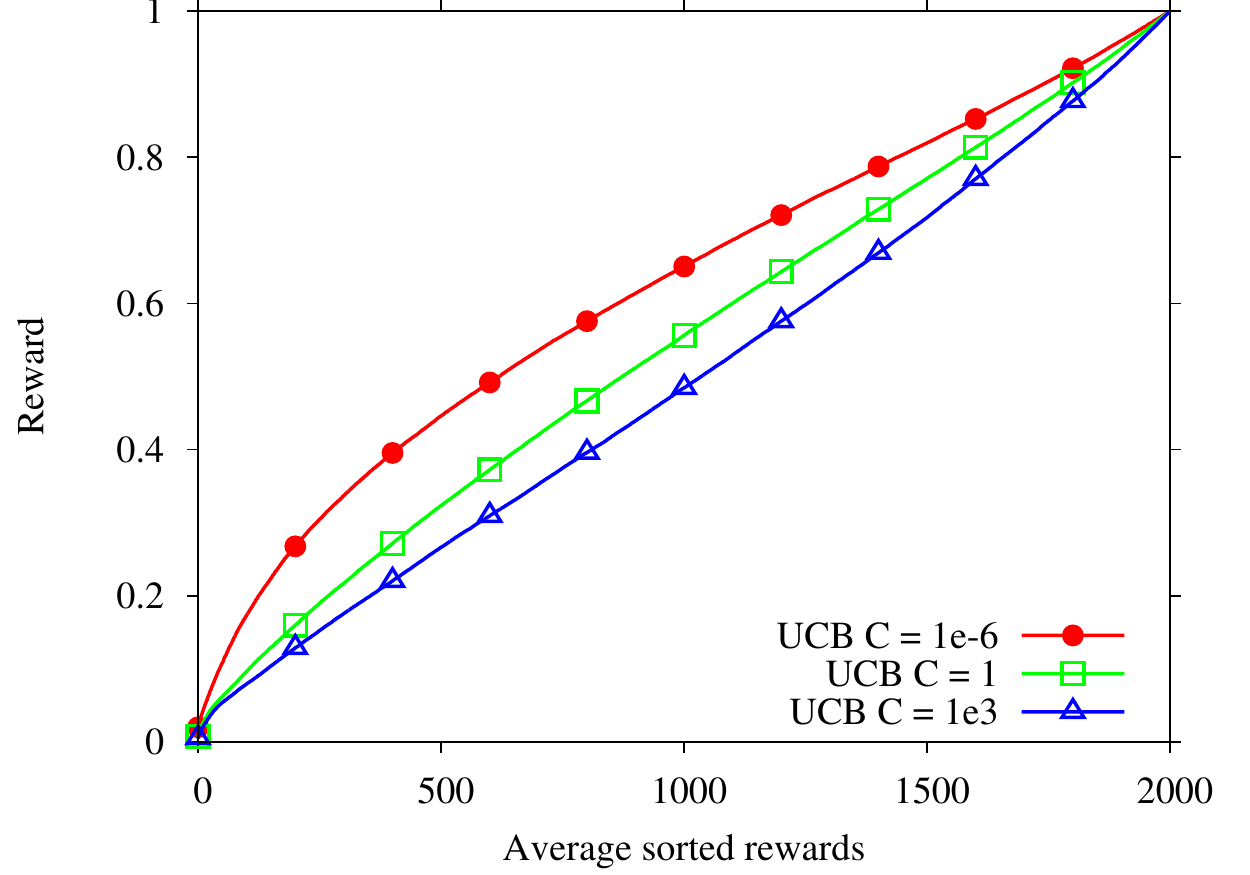} & \includegraphics[width=.4\textwidth]{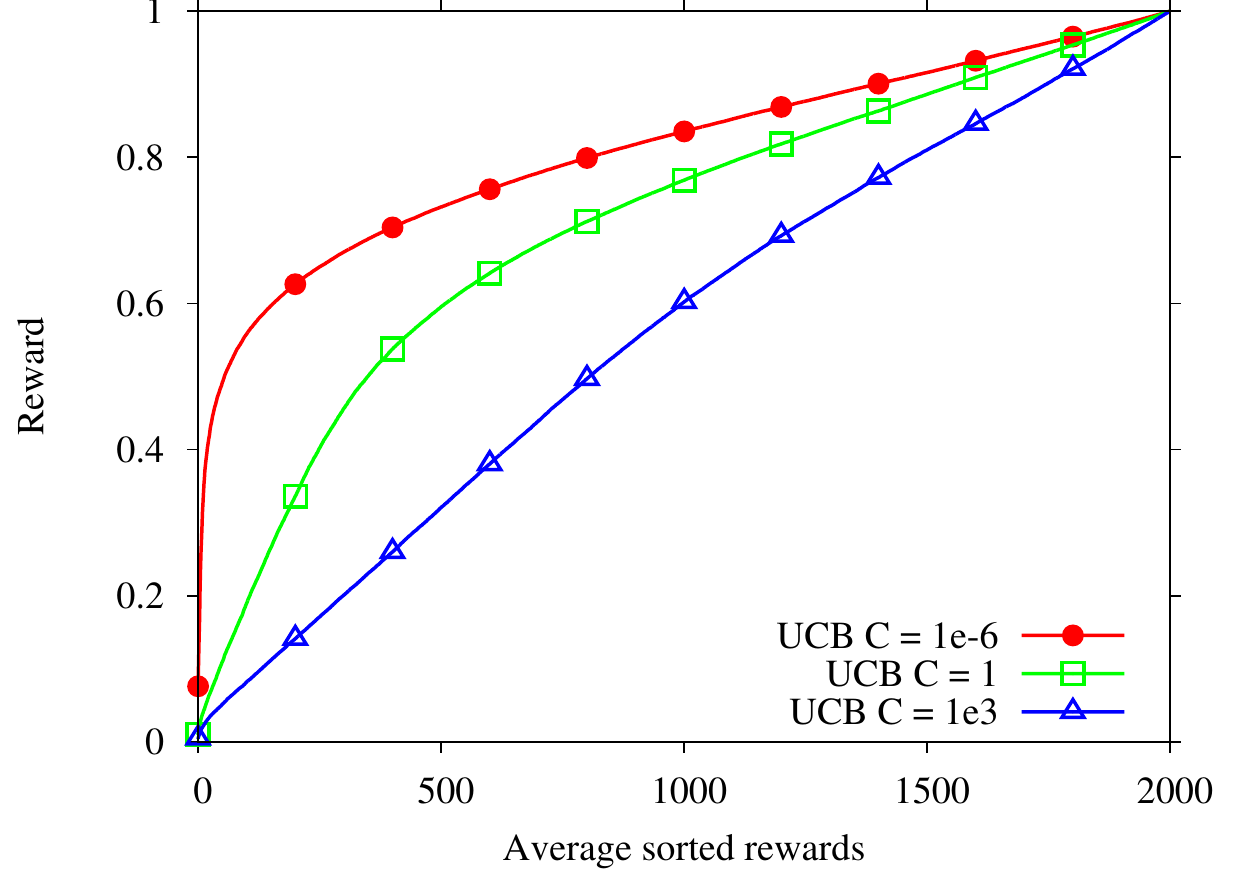}\\
\multicolumn{2}{c}{UCB with $C = 10^{i}, i = -6, 0, 3 $}\\
\includegraphics[width=.4\textwidth]{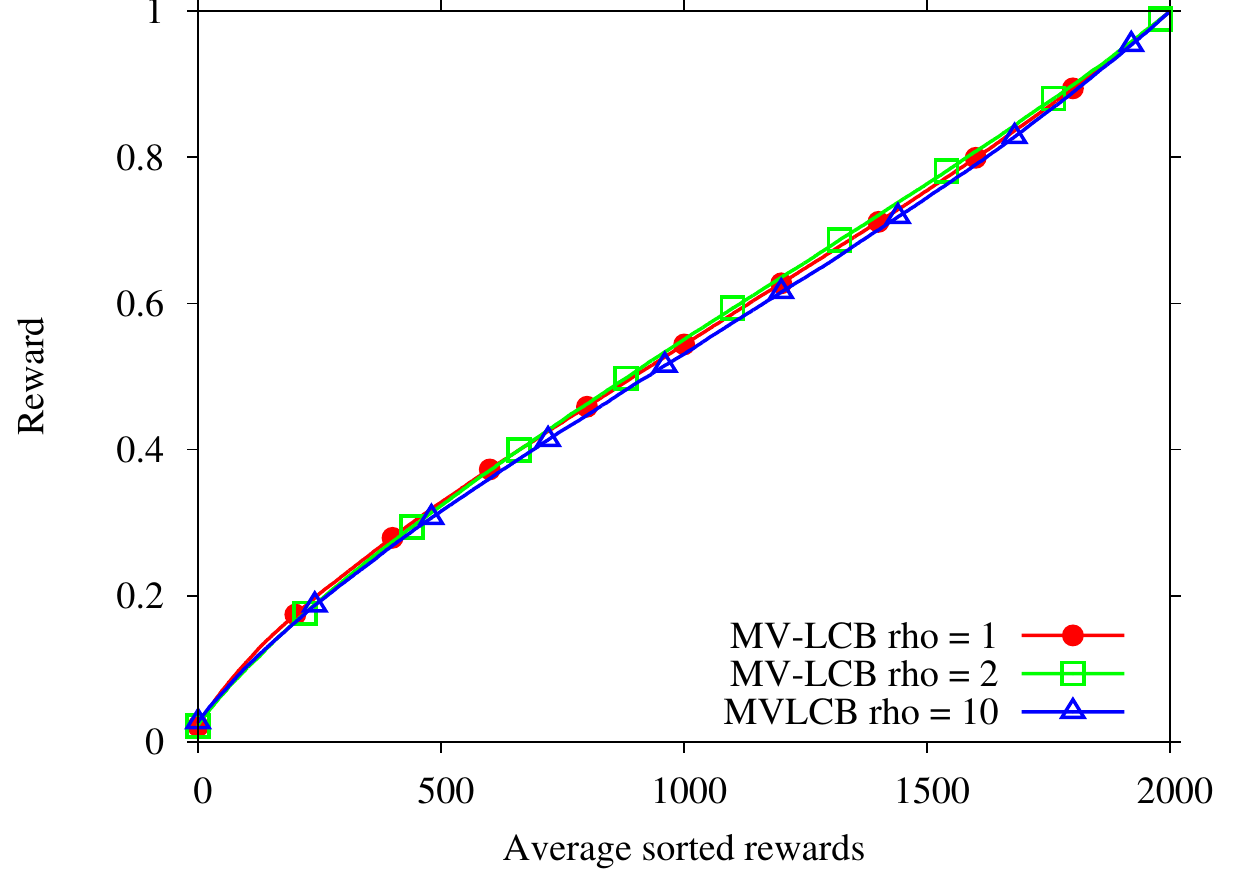} & \includegraphics[width=.4\textwidth]{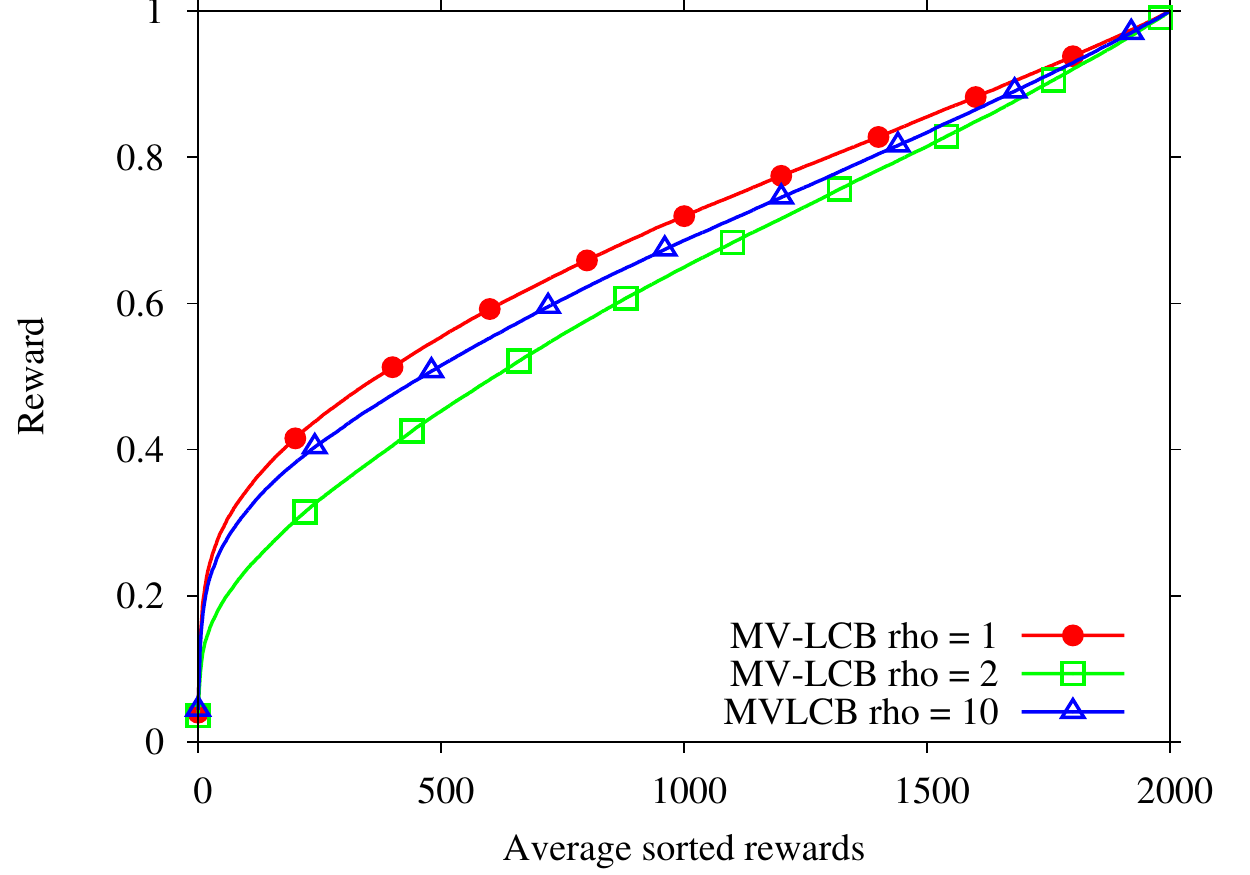}\\
\multicolumn{2}{c}{\MV\ with  $\rho \in \{ 1,2,10 \}, \delta = \frac{1}{T^2}$ }\\
\includegraphics[width=.4\textwidth]{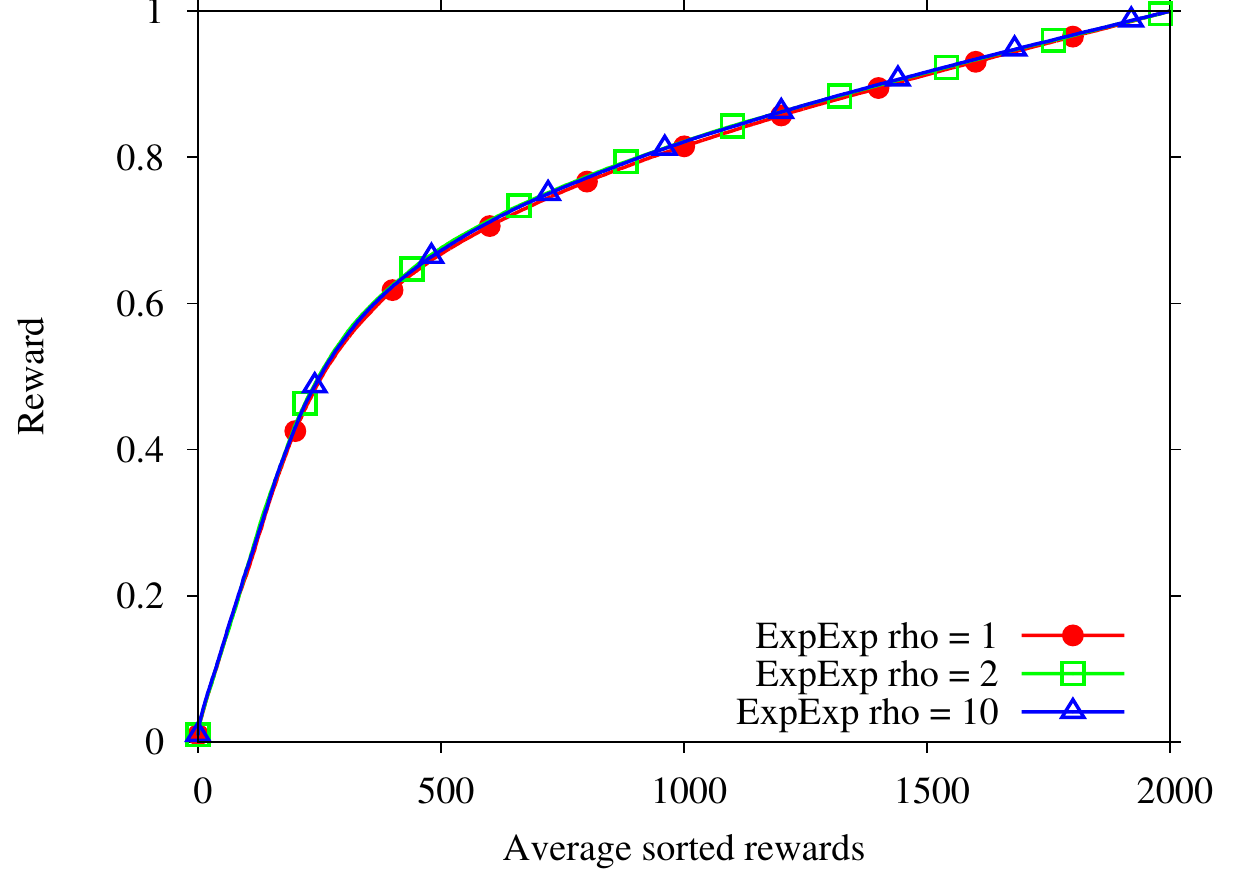} & \includegraphics[width=.4\textwidth]{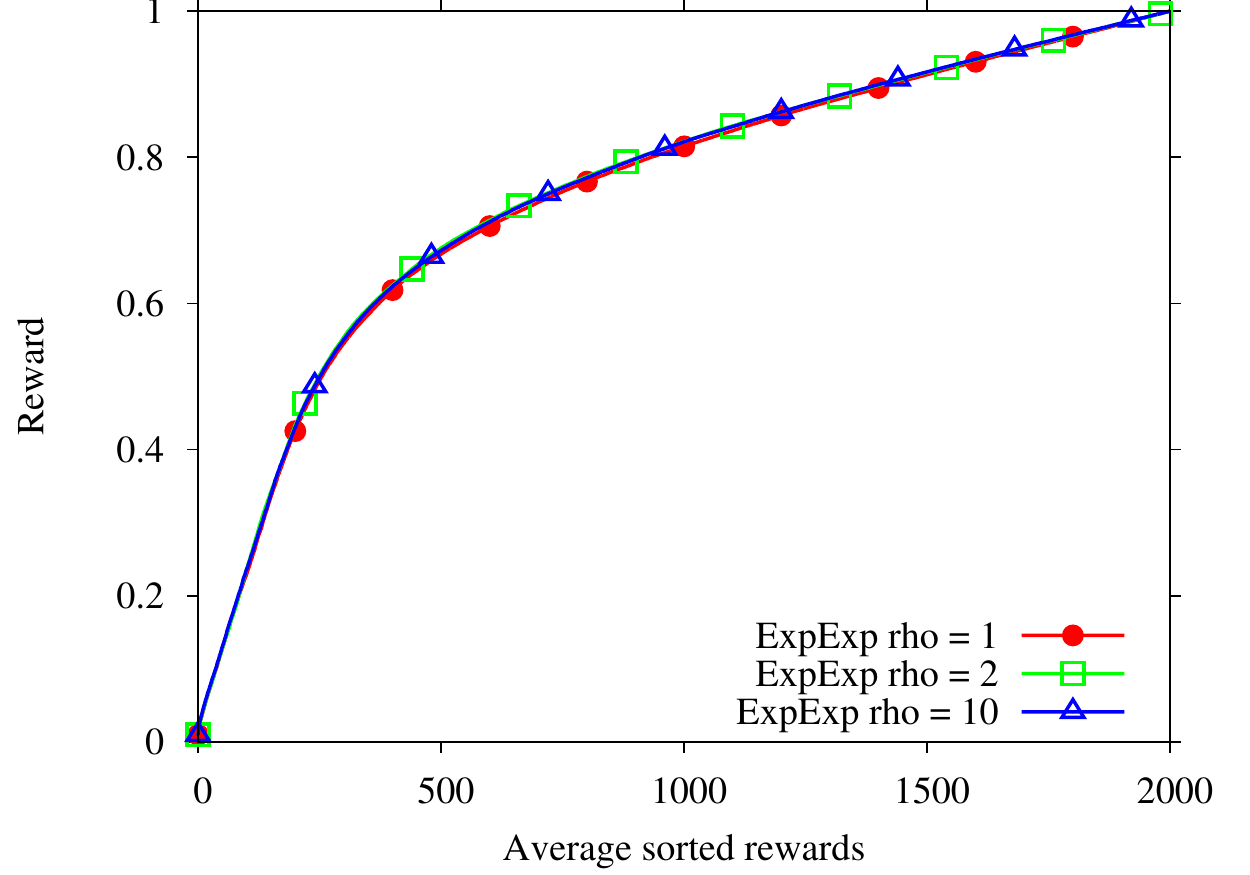}\\
\multicolumn{2}{c}{\Exp\ with $\rho \in \{ 1,2,10 \}, \tau = K(\frac{T}{14})^{2/3}$ }\\
\includegraphics[width=.4\textwidth]{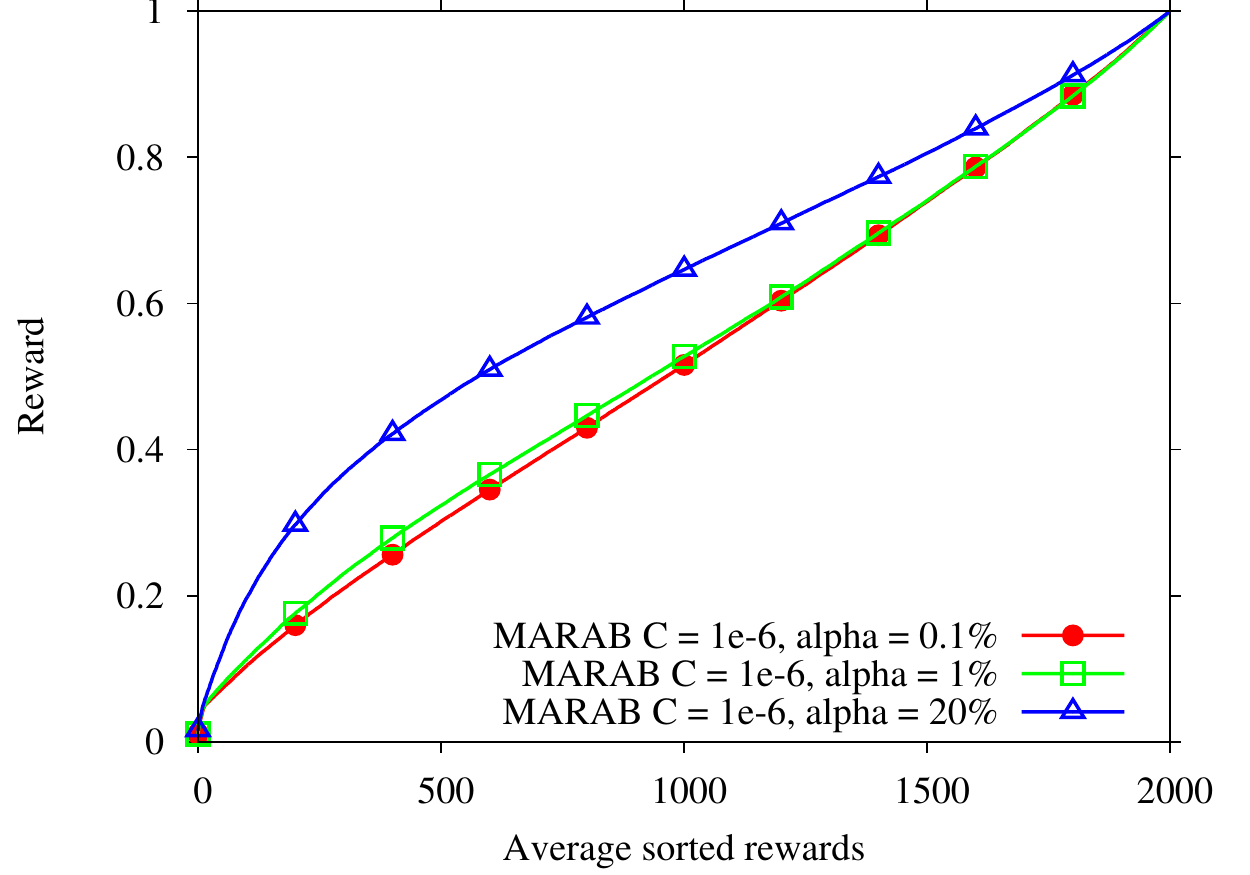} & \includegraphics[width=.4\textwidth]{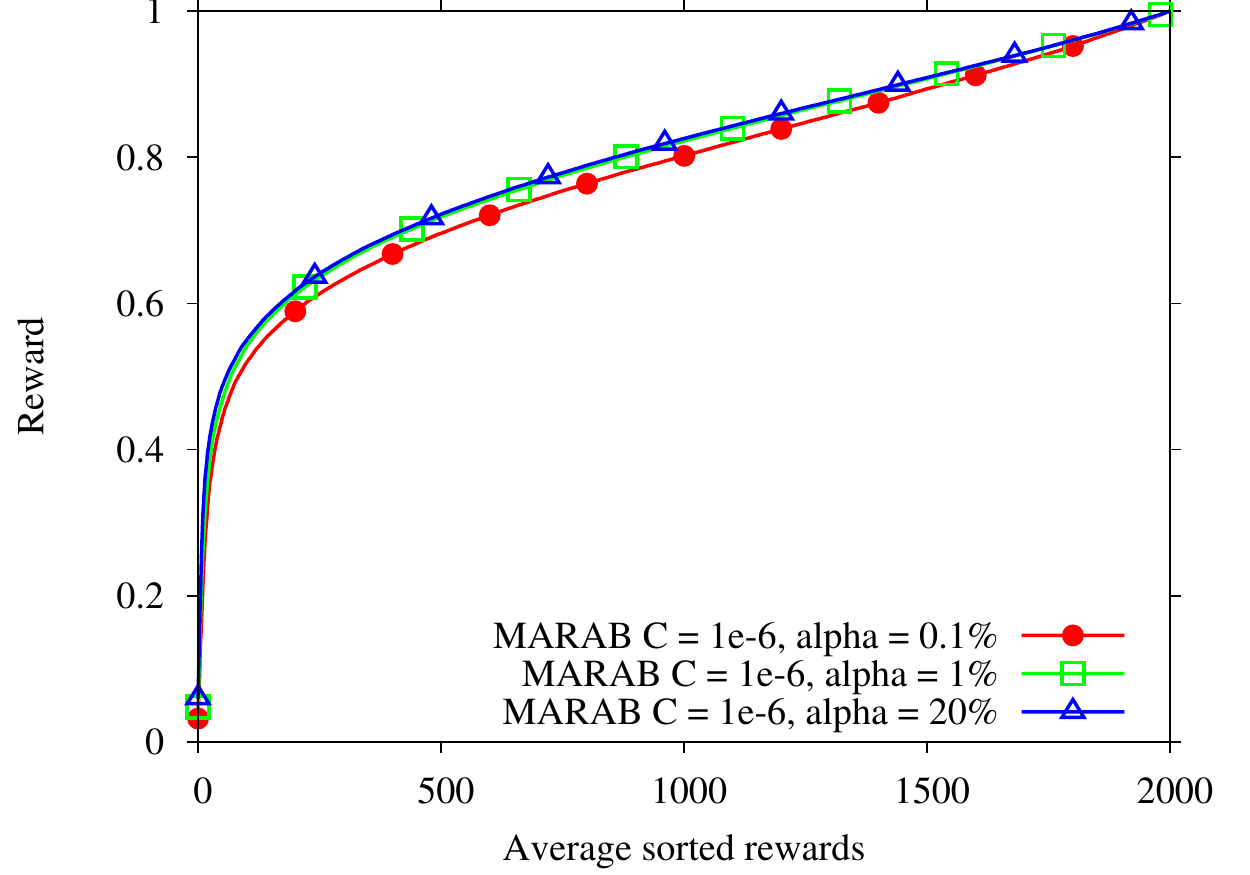}\\
\multicolumn{2}{c}{\XX{} with $C = 10^{-6}, \alpha \in \{ 0.1\%, 1\%, 20\% \}$ }\\
\end{tabular}
\end{center}
\caption{Comparative risk avoidance of UCB, \MV, \Exp\ and \XX\ on two representative artificial problems with low (left) and high (right) variance of the optimal arm. For each algorithm instant rewards averaged out of 40 runs are sorted. The time horizon is set to $T=2,000$.\label{empRewards}}
\end{figure}
The low tail of the cdf (worst average rewards gathered by the algorithm) indicates whether the algorithm actually tried poor arms.
Fig. \ref{empRewards} confirms previous results:
The noted sensitivity of UCB w.r.t. parameter $C$ unsurprisingly increases with the variance of the best arm (Fig. \ref{empRewards}, top row). 
The bad performance of 
\MV\ is confirmed; its sensitivity w.r.t. $\rho$ is low on the low variance problem as expected (Fig. \ref{empRewards}, second row, left); its sensitivity w.r.t. $\rho$ is much higher on the high variance problem (Fig. \ref{empRewards}, second row, right), with a best performance for medium values of $\rho$. \Exp\ features an excellent risk avoidance as the risky trials only take place during the exploratory phase (Fig. \ref{empRewards}, third row). 
The general robustness of \XX\ w.r.t. $C$ is confirmed; moreover, its robustness w.r.t. the risk level $\alpha$ on high variance problems is empirically shown (Fig. \ref{empRewards}, bottom row). It is seen that 
for low to medium risk ($\alpha < 20\%$), the quantile values $v_\alpha$ (section \ref{sec:defcvar}) are consistently 
higher for \XX\ than for \Exp, which is explained again from the systematic exploratory phase in \Exp. 

\subsection{Optimal energy management}\label{sec:energy}
The real-world problem motivating the presented approach is a battery management problem, where the environment is described by the  energy demand and the energy cost in each time step. The decision to be taken in each time step is a real-value $x$, determining how much energy is either used from the battery (if $x>0$) or stored in the battery (if $x < 0$). In each time step, one must meet the demand by buying $\min (0, \mbox{demand} - x)$ energy; the instant reward is the cost of the bought energy if the demand exceeds the available energy. Additionally, the battery loses some energy in each time step. 
A simplified setting is considered, where i) the energy cost is constant, the random process only dictates the energy demand in each time step; ii) 20 arms, corresponding to pre-defined strategies are considered. The strategy reward is drawn by uniform sampling with replacement from the 117 available realizations of the strategy. 

Same general trends as for the artificial problems are observed on this
real-world problem (Fig. \ref{enRewards}): i) The cumulative regret is minimal for UCB with optimally tuned $C$; ii) \MV\ is dominated by all other algorithms w.r.t. both risk avoidance and cumulative regret; iii) 
the \Exp\ regret increases linearly during the exploration phase and then reaches a plateau; iv)  \XX\ shows its good risk-avoidance ability regardless of the $C$ value, and MIN yields same results. Overall, \XX\ suffers a slight regret increase compared to UCB at its best, with a slightly better reward cdf in the region of low rewards. 

\begin{figure}[htbp]
\begin{center}
\begin{tabular}{cc}
\includegraphics[width=.4\textwidth]{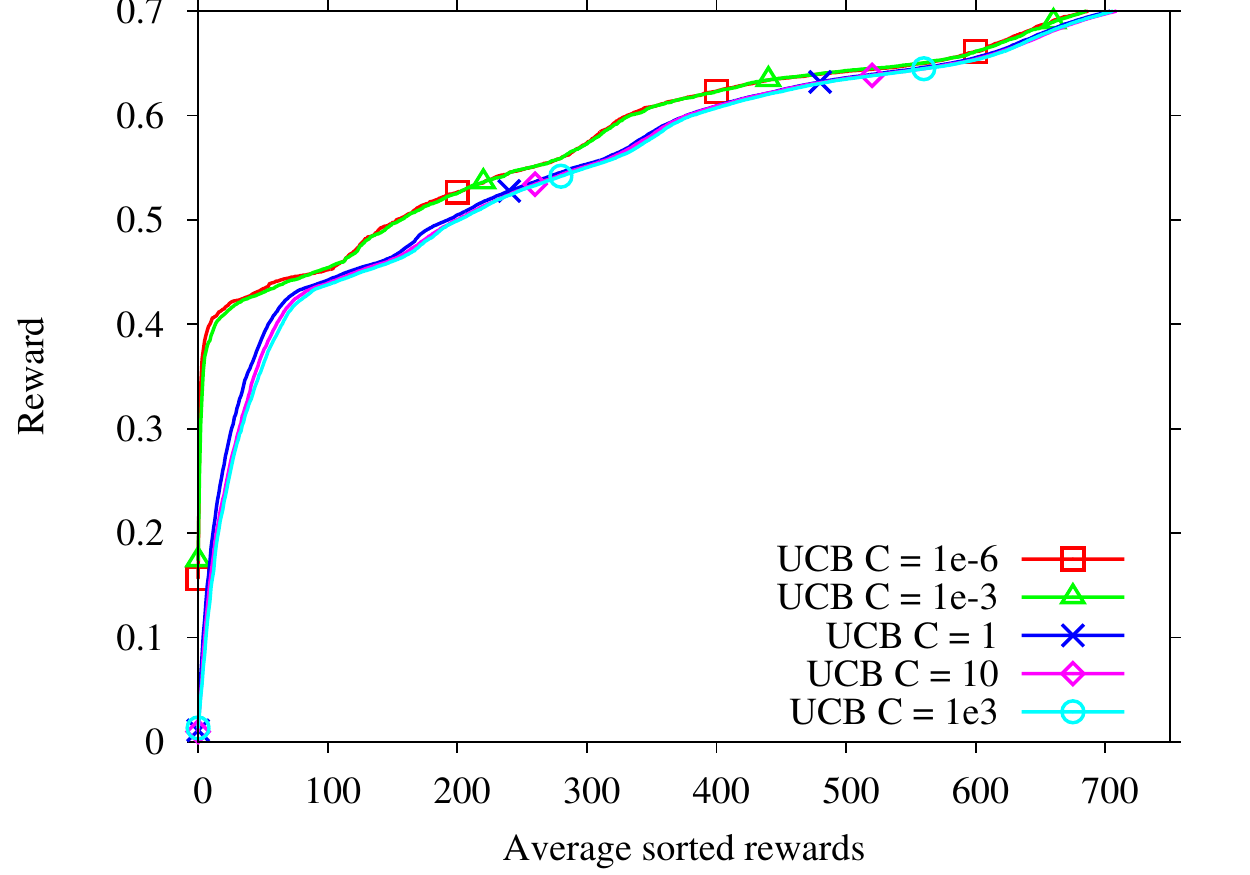} & \includegraphics[width=.4\textwidth]{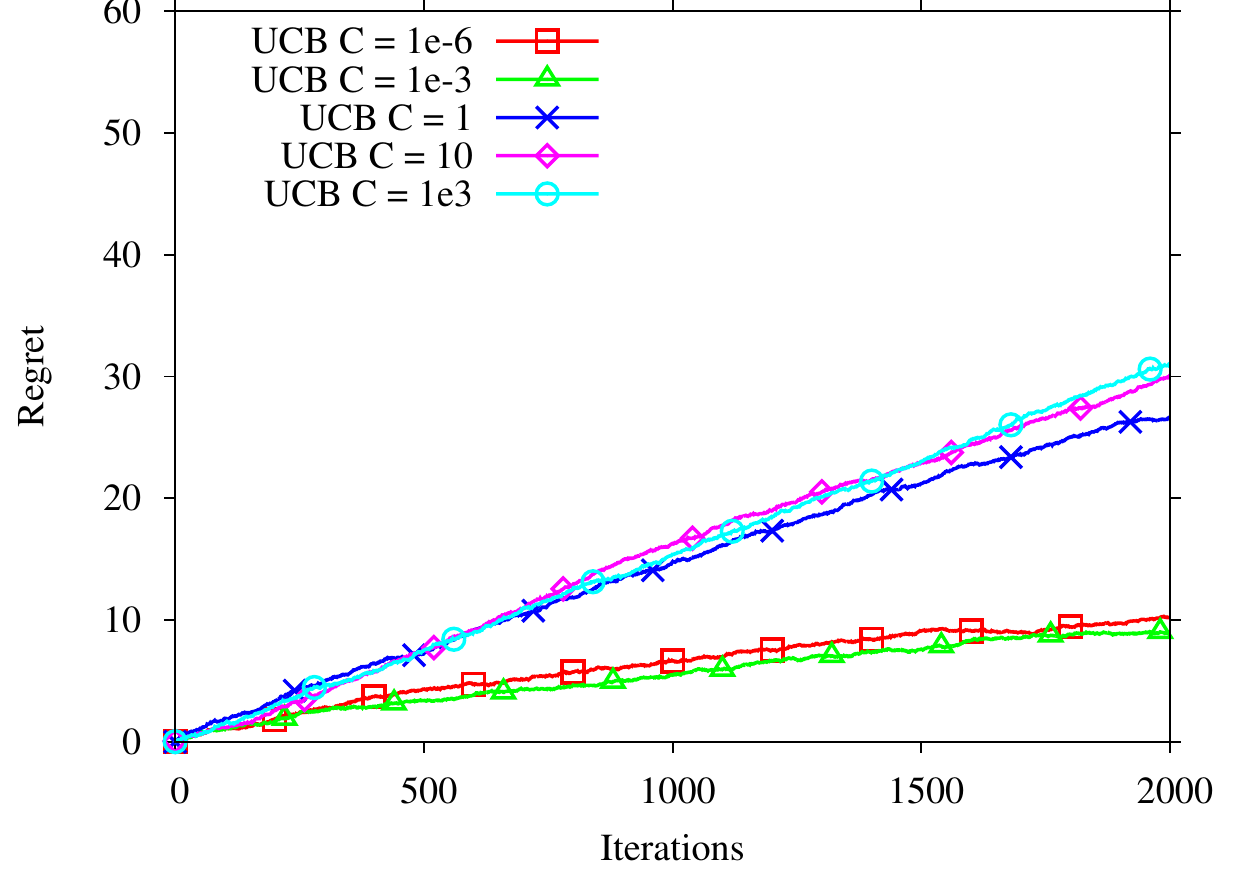}\\
\multicolumn{2}{c}{UCB. $C \in \{ 10^{-6}, 10^{-3}, 1, 10, 10^3 \}$}\\
\includegraphics[width=.4\textwidth]{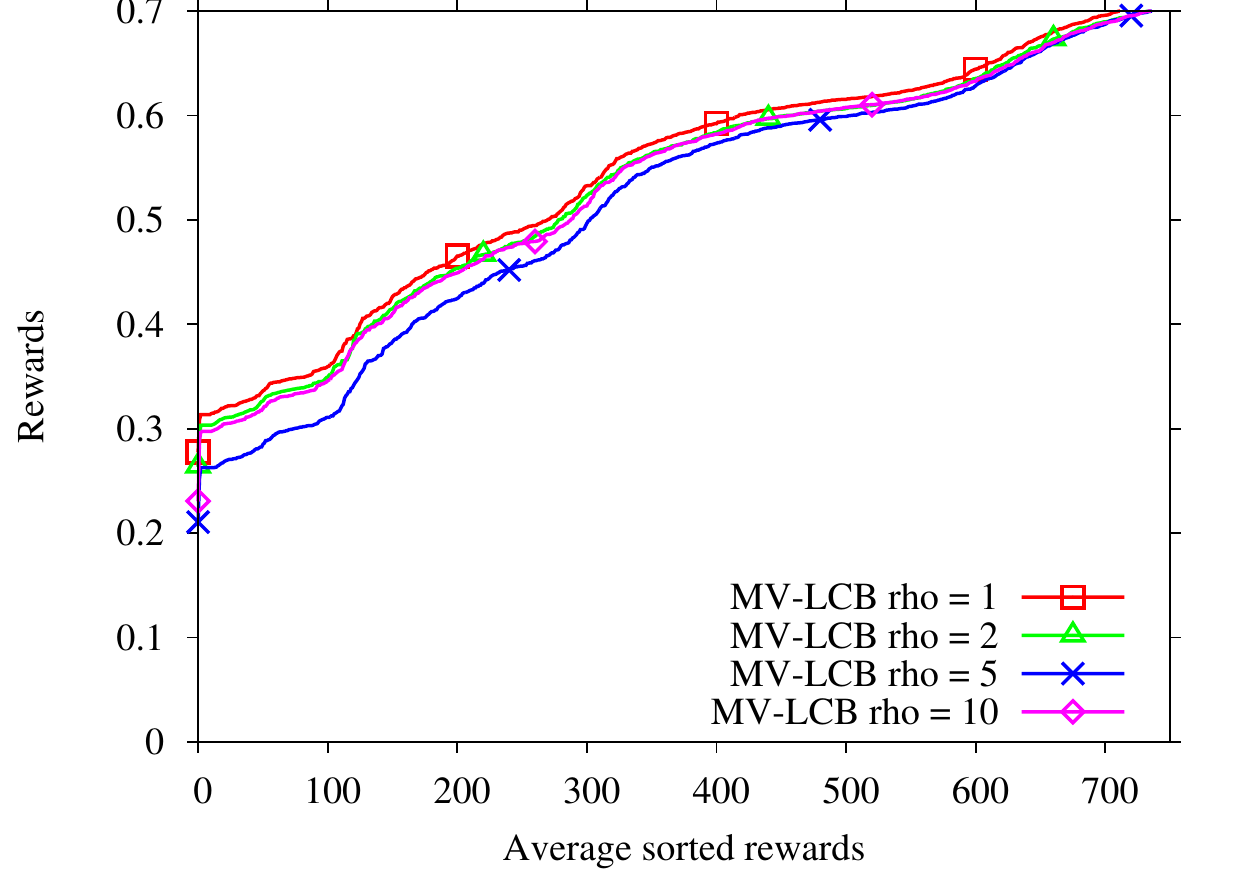} & \includegraphics[width=.4\textwidth]{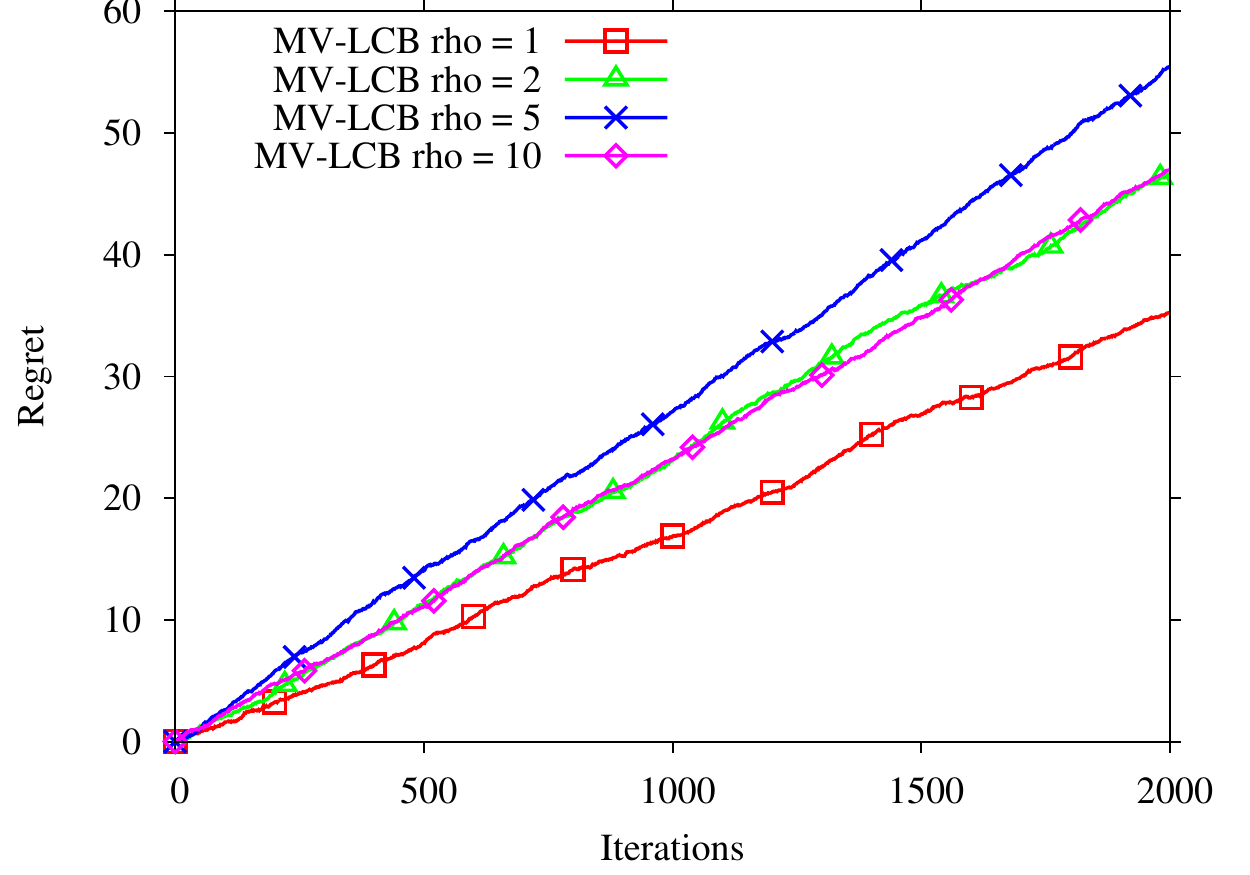}\\
\multicolumn{2}{c}{MV-LCB $\rho \in \{ 1,2, 5,10 \}, \delta = \frac{1}{T^2}$ }\\
\includegraphics[width=.4\textwidth]{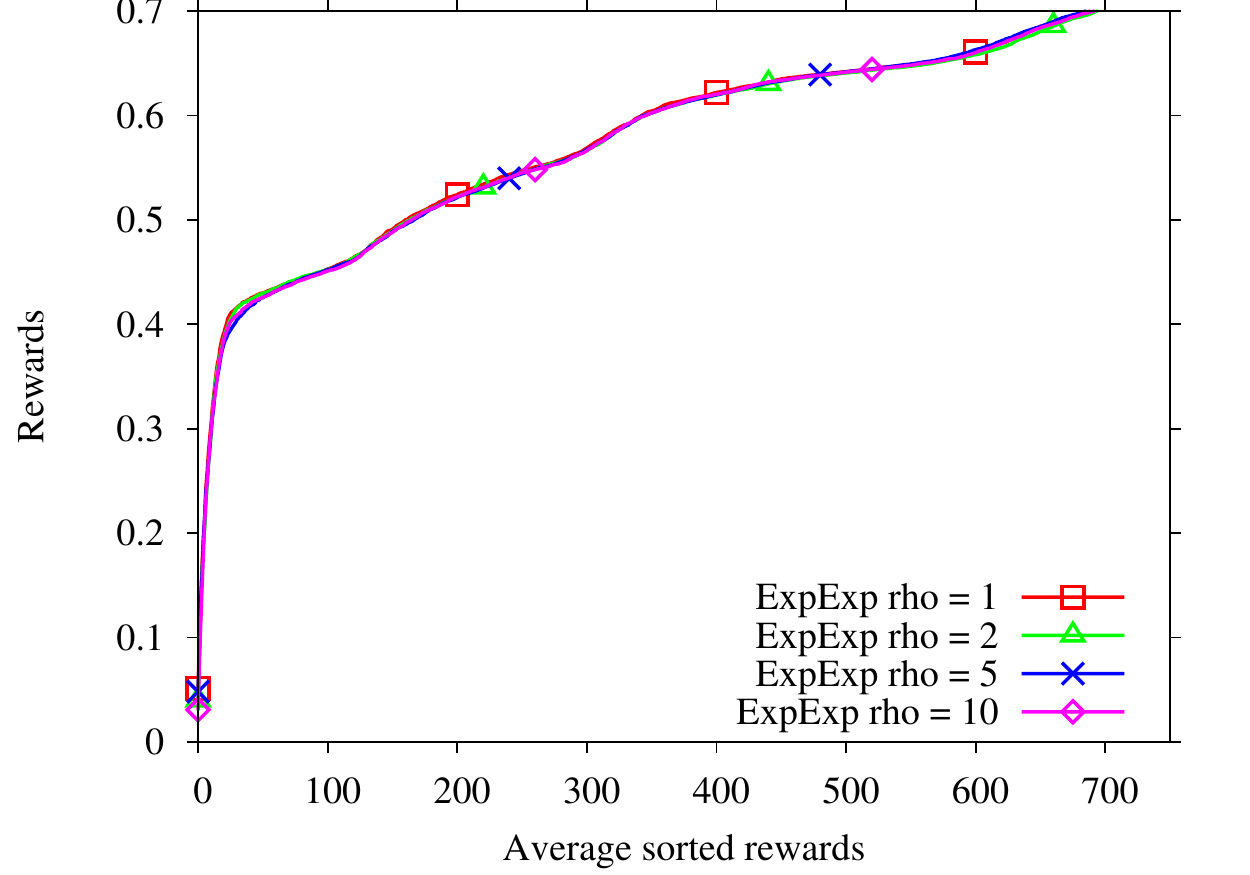} & \includegraphics[width=.4\textwidth]{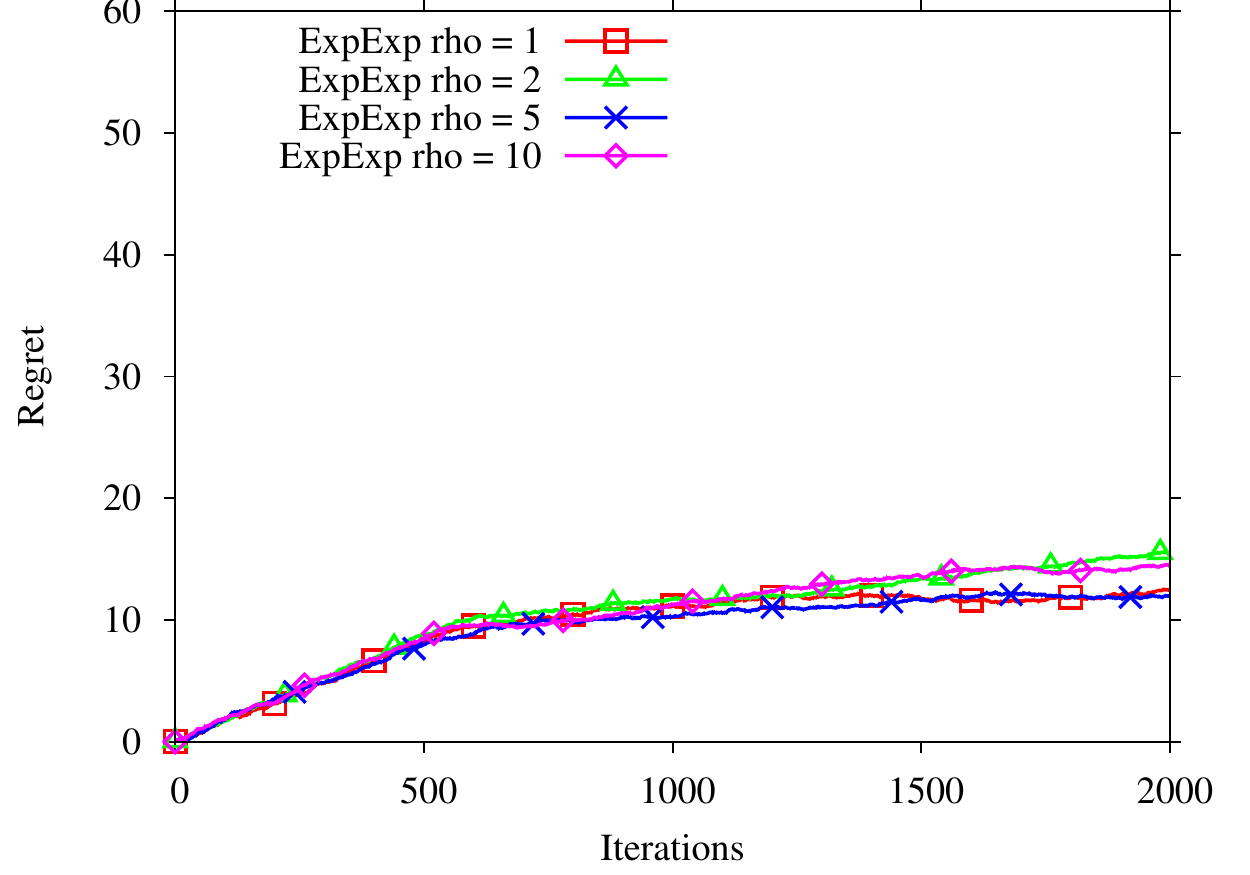}\\
\multicolumn{2}{c}{ExpExp $\rho \in \{ 1,2, 5,10 \}, \tau = K(\frac{T}{14})^{2/3}$ }\\
\includegraphics[width=.4\textwidth]{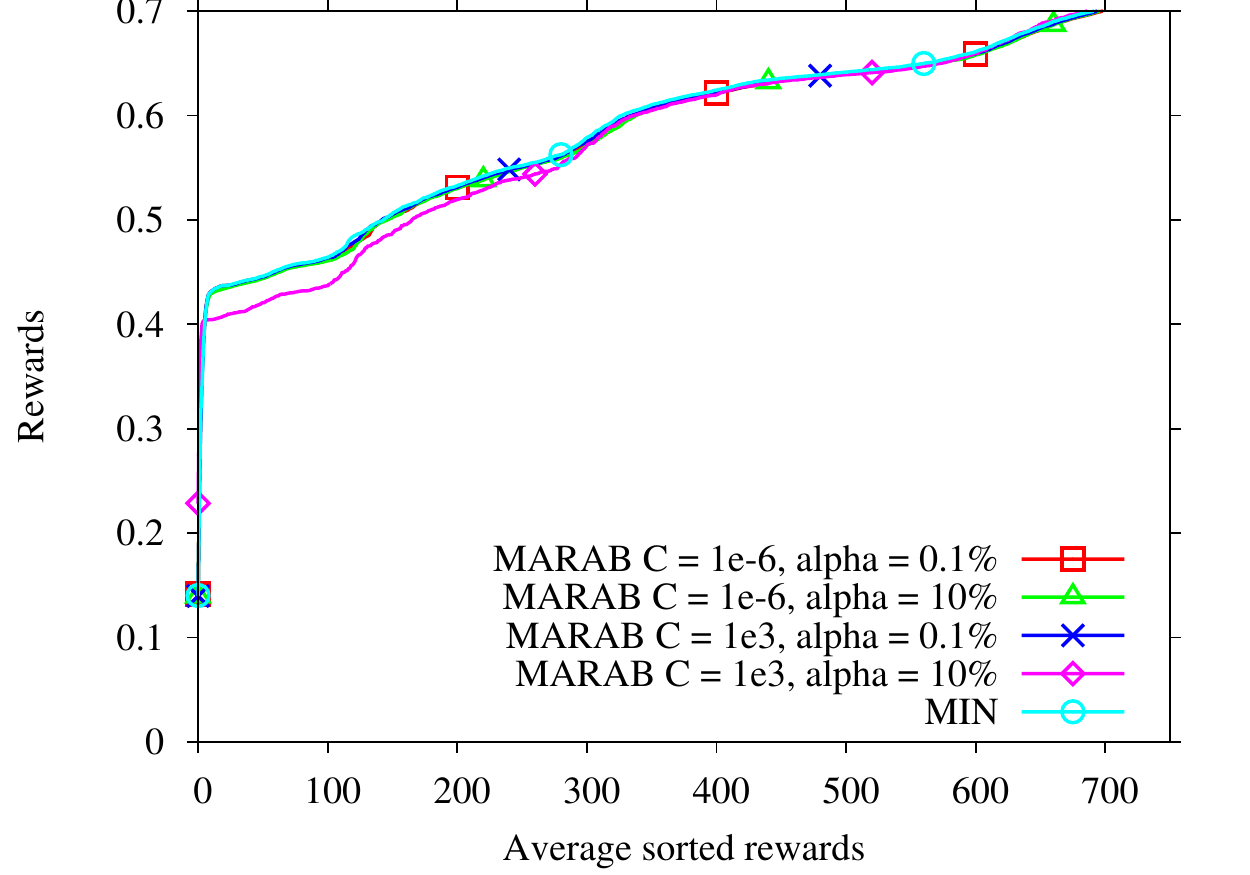} & \includegraphics[width=.4\textwidth]{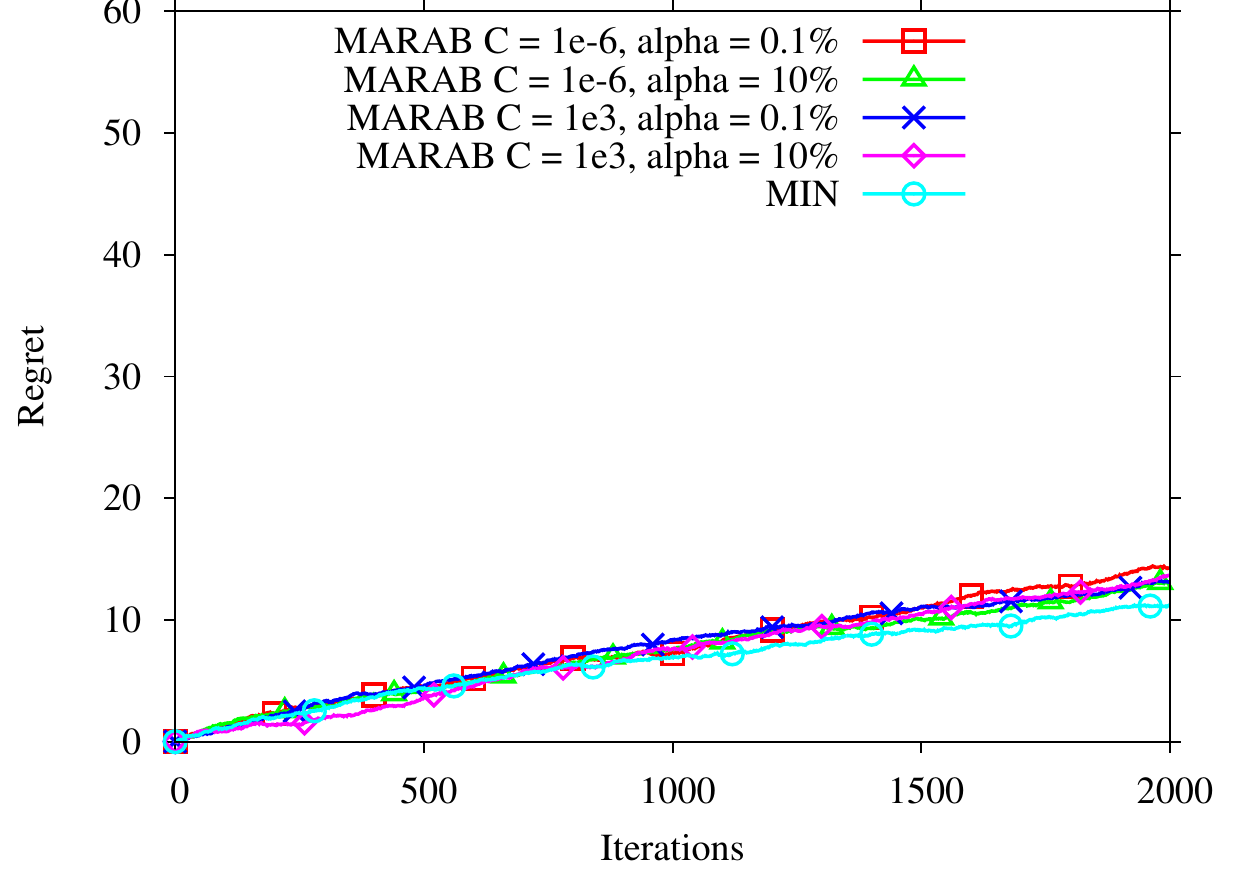}\\
\multicolumn{2}{c}{\XX{} $C \in \{ 10^{-6}, 10^3\}, \alpha \in \{ 0.1\%, 10\% \}$ }\\
\end{tabular}
\end{center}
\caption{Comparative performance of UCB, \MV, \Exp and \XX\ on a real-world energy management problem. Left: sorted instant rewards (truncated to the 37.5\% worst cases for readability). Right: empirical cumulative regret with time horizon $T = 100K$, averaged out of 40 runs.}
\label{enRewards}
\end{figure}

%% file: toy_camera.tex
\subsection{Proof of concept}\label{sec:toyex}
An ad-hoc problem generator satisfying the assumptions done in Prop. \ref{prop44} is used to compare MIN, UCB and \XX\ in the favorable case. 
\begin{figure}[htbp]
\begin{center}
 \begin{tabular}{cc}
 \includegraphics[width=.5\textwidth]{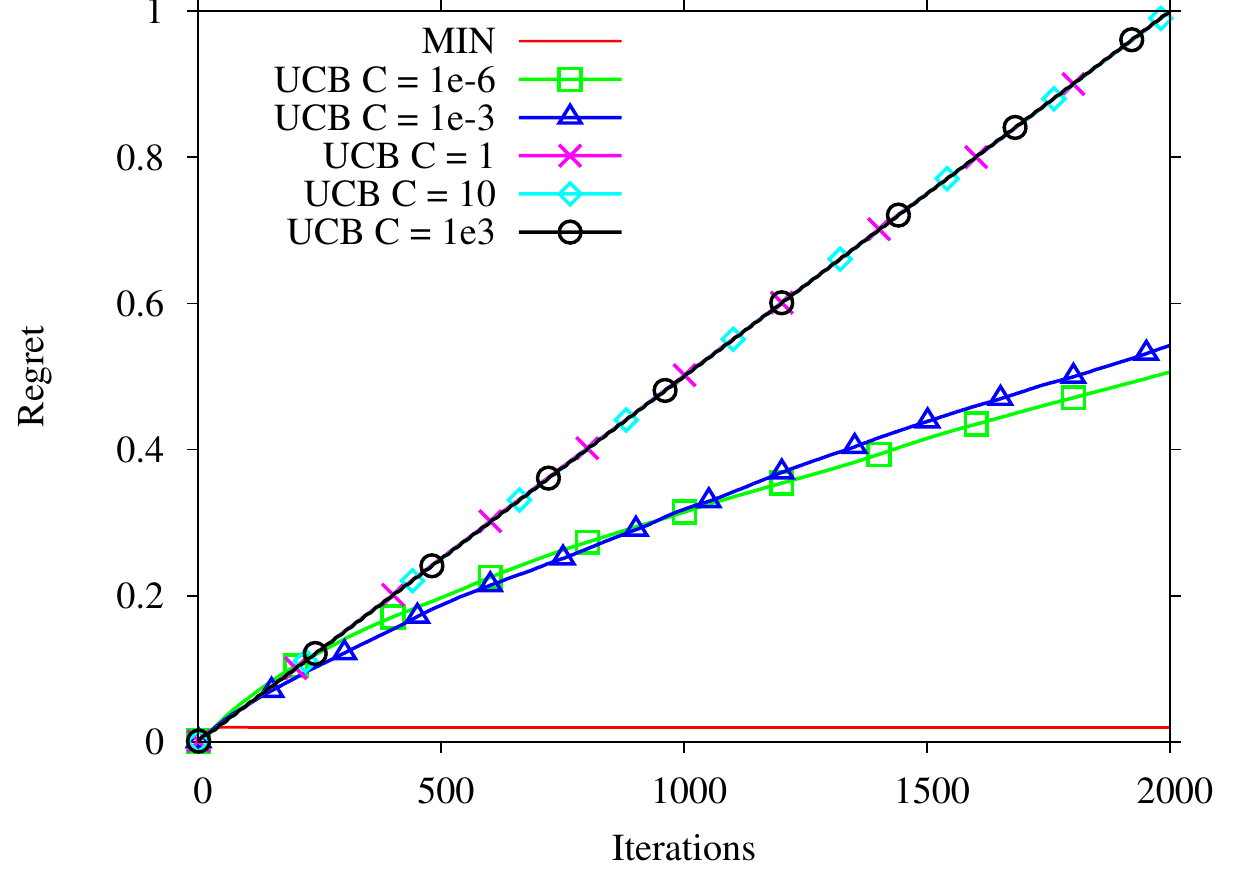} & 
 \includegraphics[width=.5\textwidth]{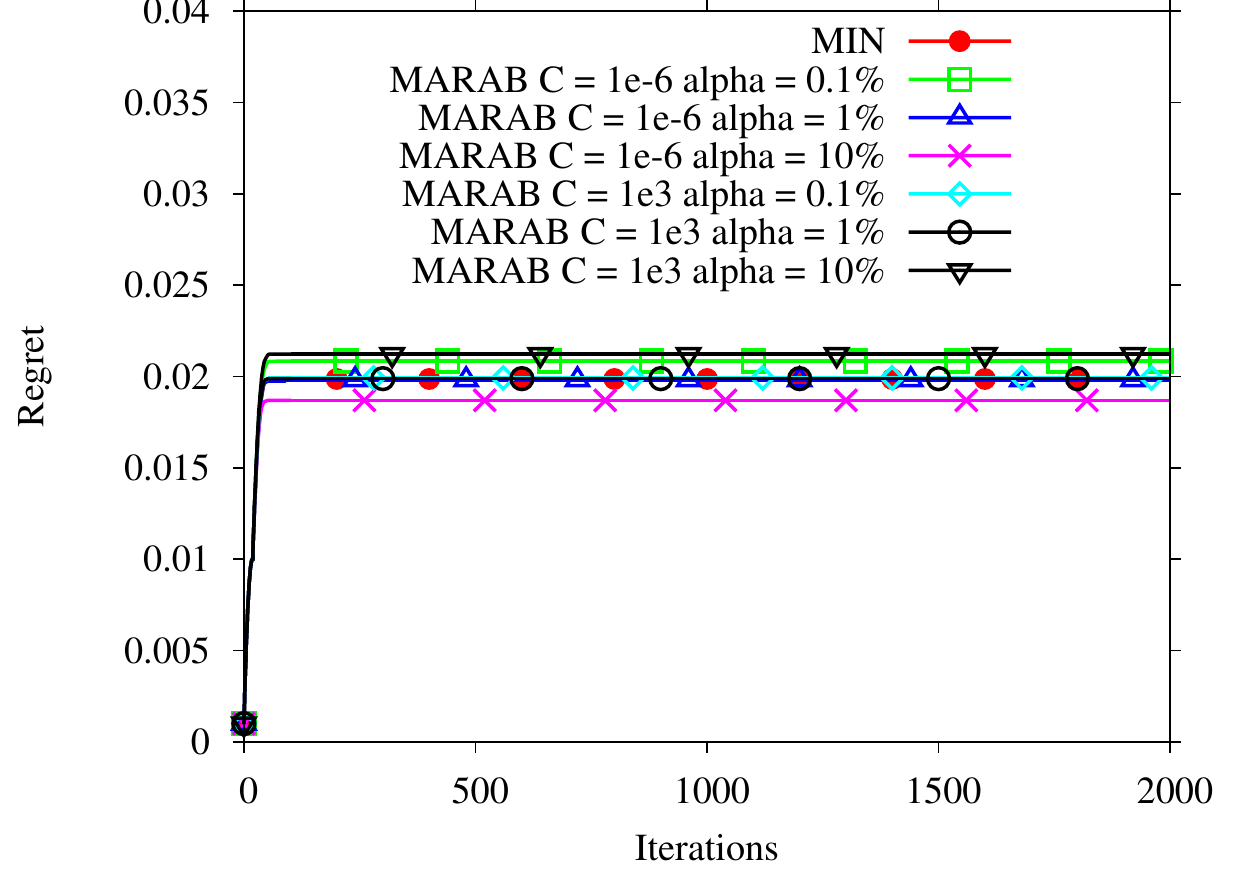}\\
(a) MIN and UCB  & (b) MIN and \XX\ \\
$C = 10^{i}, ~i=-6\ldots 3$& $C = 10^{i}, ~i=-6\ldots 3$, $\alpha = .1 \%, 1\%, 10\%$) \\
 \end{tabular}
\end{center}
\caption{\label{figtoyex} Theoretical cumulative regret of UCB, MIN and \XX\ under the assumptions of Prop. \ref{prop44},  averaged out of 40 runs. Parameter $C$ ranges in $\{10^{i}, i = -6 \ldots 3\}$. Risk quantile level $\alpha$ ranges from .1\% to 10\%.\\
Left: UCB regret increases logarithmically %(linearly if $C$ is ill-tuned)
 with the number of iterations for well-tuned $C$; MIN identifies the best arm after 50 iterations and its regret is constant thereafter.
%%Left: UCB regret increases logarithmically (linearly if $C$ is ill-tuned) with the number of iterations; MIN identifies the best arm after circa ZZ iterations and its regret is constant thereafter. \\
Right: zoom on the lower region of Left, with MIN and \XX\ regrets; \XX\ regret is close to that of MIN, irrespective of the $C$ and $\alpha$ values in the considered ranges. }
\end{figure}

Each problem involves 20 arms. 
The $i$-th arm distribution $\nu_i$ is set to a uniform distribution on a segment in $[0,1]$, centered on $\mu_i$ with radius $r_i$ ($\nu_i = \mathcal{U}([\mu_i - r_i, \mu_i + r_i])$). Mean $\mu_i$ (respectively radius $r_i$) decreases (resp. increases ) with $i$. The mean-related and minimum-related margins are respectively controlled from two generative parameters\footnote{With $\Delta_{max}$ and $r_{max}$ two generative parameters, $\mu_i$ is a decreasing affine function of $i$, 
$\mu_i = \mu^* - \frac{i -1}{(K-1)} \, \Delta_{\max}$. \\
$r_i$ is an increasing affine function of $i$, with $r_1 = \mu^* - a^*$ and $r_i = r_1 + \frac{i -1}{K-1} \, r_{\max}$.\\
The mean-related margin $\Delta_{\mu,i}$ is thus controlled from $\Delta_{max}$; the min-related margin $\Delta_{a,i}$ is controlled from $\Delta_{max}$ and $r_{max}$, in such a way that $\Delta_{a,i} > \Delta_{\mu,i}$.
}.
The theoretical cumulative regrets of UCB, MIN and \XX{} are displayed in Fig. \ref{figtoyex} (averaged out of 40 independent runs with $\mu^* = 0.5, a^* = \mu^* - 10^{-3}$ and maximal radius $0.5$). Parameter $C$ of UCB and \XX\ 
ranges in $\{10^{i}, i = -6 \ldots 3\}$ and the risk level $\alpha$ ranges 
from .1\% to 10\%. 
By construction, this artificial problem favors MIN against UCB; firstly it satisfies the assumptions of Prop. \ref{prop44}; moreover since distributions $\nu_i$ are uniform, $A \geq 1$.
In this easy setting, MIN catches the best arm after 50 iterations
and yields a constant regret thereafter (no exploration). \XX\ features the same behavior for a wide range of values of $C$ and $\alpha$; its very low sensitivity w.r.t. $C$ slightly increases for high values of $\alpha$ ($\alpha > 20\%$). 
The disappointing UCB performance is blamed on the high variance of the worse arms, slowing down the accurate estimation of their mean.